\documentclass[a4paper,12pt]{article}

\usepackage[utf8]{inputenc}
\usepackage[unicode=true,hypertexnames=false]{hyperref}
\usepackage{amsmath,amssymb,amsthm,amsfonts,mathtools}
\usepackage[linesnumbered,ruled]{algorithm2e}
\usepackage{xspace}
\usepackage{xcolor}
\usepackage{dsfont}
\usepackage{multirow}

\usepackage{pgfplots}
\pgfplotsset{compat=newest}
\pgfmathdeclarefunction{f1}{1}{%
  \pgfmathparse{1/#1}%
}
\pgfmathdeclarefunction{f2}{1}{%
  \pgfmathparse{#1^(1.5)}%
}
\pgfmathdeclarefunction{f3}{1}{%
  \pgfmathparse{(#1 - 1)^(1.5)}%
}
\pgfplotscreateplotcyclelist{myplotcycle}{%
blue,dashed\\     
blue\\            
red!80!black\\    
violet!80!black\\ 
orange\\          
green!70!black\\  
cyan!80!black\\   
black,dashed\\    
black\\           
magenta\\         
}

\pgfplotscreateplotcyclelist{myplotcycle2}{%
blue,dashed\\     
red!80!black\\    
violet!80!black\\ 
orange\\          
green!70!black\\  
cyan!80!black\\   
black,dashed\\    
black\\           
red!80!black,dotted\\    
violet!80!black,dotted\\ 
orange,dotted\\          
green!70!black,dotted\\  
blue,dotted\\            
}

\pgfplotscreateplotcyclelist{myplotcycle3}{%
black,dashed\\    
red!80!black\\    
violet!80!black\\ 
orange\\          
green!70!black\\  
cyan!80!black\\   
}

\newcommand{\oea}{$(1 + 1)$~EA\xspace}

\newcommand{\oplea}{$(1 + \lambda)$~EA\xspace}
\newcommand{\mpoea}{$(\mu + 1)$~EA\xspace}

\newcommand{\ollea}{$(1 + (\lambda , \lambda))$~GA\xspace}
\newcommand{\onemax}{\textsc{OneMax}\xspace}
\newcommand{\maxsat}{\textsc{MAX-3SAT}\xspace}
\newcommand{\leadingones}{\textsc{LeadingOnes}\xspace}
\newcommand{\jump}{\textsc{Jump}\xspace}
\newcommand{\N}{{\mathbb N}}
\newcommand{\R}{{\mathbb R}}

\DeclareMathOperator{\Bin}{Bin}
\DeclareMathOperator{\lnp}{lnp}

\DeclareMathOperator*{\argmax}{arg\,max}

\newtheorem{theorem}{Theorem}
\newtheorem{lemma}[theorem]{Lemma}

\makeatletter 
\renewcommand*{\@fnsymbol}[1]{\ensuremath{\ifcase#1\or\dagger\or*\else\@arabic{#1}\fi}}
\makeatother

\title{Fast Mutation in Crossover-based Algorithms\thanks{Extended version of the paper~\cite{AntipovBD20gecco} in the proceedings of GECCO. This version contains all proofs and other details that had to be omitted in the conference version for reasons of space. Also, we have greatly expanded the experimental section.}}

\author{Denis Antipov\thanks{Corresponding author} \\
		ITMO University \\
  		St. Petersburg, Russia \\
  		and \\
		Laboratoire d'Informatique (LIX), \\
		CNRS, \'Ecole Polytechnique, \\ 
		Institut Polytechnique de Paris \\
		Palaiseau, France \\
		antipovden@yandex.ru$^*$ \\
		\and
		Maxim Buzdalov\\
		ITMO University \\
		St. Petersburg, Russia \\
		mbuzdalov@gmail.com \\
		\and
		Benjamin Doerr \\
		Laboratoire d'Informatique (LIX), \\
		CNRS, \'Ecole Polytechnique, \\ 
		Institut Polytechnique de Paris \\
		Palaiseau, France \\
		doerr@lix.polytechnique.fr \\
} 

\begin{document}

\maketitle

{\sloppy

\begin{abstract}
	The heavy-tailed mutation operator proposed in Doerr, Le, Makhmara, and Nguyen (GECCO 2017), called \emph{fast mutation} to agree with the previously used language, so far was proven to be advantageous only in mutation-based algorithms. There, it can relieve the algorithm designer from finding the optimal mutation rate and nevertheless obtain a performance close to the one that the optimal mutation rate gives. 
	
	In this first runtime analysis of a crossover-based algorithm using a heavy-tailed choice of the mutation rate, we show an even stronger impact. For the $(1+(\lambda,\lambda))$ genetic algorithm optimizing the \onemax benchmark function, we show that with a heavy-tailed mutation rate a linear runtime can be achieved. This is asymptotically faster than what can be obtained with any static mutation rate, and is asymptotically equivalent to the runtime of the self-adjusting version of the parameters choice of the $(1+(\lambda,\lambda))$ genetic algorithm. This result is complemented by an empirical study which shows the effectiveness of the fast mutation also on random satisfiable \maxsat instances.
\end{abstract}

\section{Introduction}

It is often cited as a strength of evolutionary algorithms (EAs) that by setting the parameters right the algorithm can be adjusted to the particular problem to be solved. However, it is also known that this process of optimizing the parameters is time-consuming and needs a lot of expert knowledge.

The theoretical research in this field (see, e.g.,~\cite{AugerD11, DoerrN20, Jansen13, NeumannW10}) has contributed to this challenge via mathematical runtime analyses for general parameter values, which allow to understand the influence of the parameter on the performance and allow  to derive optimal parameter values. Examples include (i)~the works of Jansen, de Jong, and Wegener~\cite{JansenJW05} as well as Doerr and K\"unnemann~\cite{DoerrK15}, which determine the runtime of the \oplea on \onemax for general value of $\lambda$ and from this conclude that a linear speed-up with regard to the number of iterations exists only for $\lambda = O\big(\frac{\log(n) \log\log(n)}{\log\log\log(n)}\big)$, (ii)~Witt's analysis~\cite{Witt06} of the runtime of the \mpoea for general values of $\mu$ on the \leadingones benchmark, which in particular shows that for $\mu = O(\frac{n}{\log n})$ a larger parent population does not lead to an asymptotic slow-down of the algorithm, or (iii)~the results of Lehre~\cite{Lehre10,Lehre11} and many follow-up works, which for many non-elitist algorithms determine asymptotically precise thresholds for the selection pressure that separate a highly inefficient regime from one with polynomial runtimes. 

Concerning the mutation rate $p$ of the standard bit mutation operator for bit strings of length $n$, which is our main object of interest in this work, a large number of classic results suggests that a value of $p=\frac 1n$ or close by is a good choice. We note that a mutation rate of $p=\frac 1n$ means that on average a single bit is flipped. The recommendation $p=\frac 1n$ can already be found in~\cite{Back93,Muhlenbein92}. Rigorously proven results show, among others, that only $p = \Theta(\frac 1n)$ can give an $O(n \log n)$ runtime of the \oea on \onemax~\cite{GarnierKS99}, that the asymptotically optimal mutation rate for the \oea on \leadingones is approximately $p=\frac{1.59}{n}$, that $p = (1 \pm o(1)) \frac 1n$ is the asymptotically best mutation rate of the \oea for all pseudo-Boolean linear functions~\cite{Witt13}, that only a mutation rate below $\frac cn$, where $c$ is a specific constant, guarantees a polynomial runtime of the \oea on all monotonic functions~\cite{DoerrJSWZ13,Lengler18}, and that $(1 \pm o(1)) \frac 1n$ is the optimal mutation rate for the \oplea on \onemax when $\lambda$ is small~\cite{GiessenW17}. 

In the light of this previous state of the art, it came as a surprise when Doerr, Le, Makhmara, and Nguyen~\cite{DoerrLMN17} determined the runtime of the \oea on jump functions for general mutation rates and observed that here much higher mutation rates were optimal\footnote{As a reviewer of~\cite{AntipovBD20gecco} pointed out, in~\cite{Prugel04} an upper bound was shown for the runtime of the \oea with general mutation rate on the hurdle problem with hurdle width $2$ and $3$. This upper bound is minimized by the mutation rates $\frac 2n$ and $\frac 3n$. This could have been seen earlier as a hint that larger mutation rates can be useful. Since the central research question discussed in~\cite{Prugel04} was whether crossover is beneficial or not, apparently this detail was overlooked by the broader scientific audience.}. The jump function $\jump_{nk}$ (we deviate here from the notation of~\cite{DoerrLMN17}) is a function defined on bit-string of length $n$ which is mostly identical to the easy \onemax function, but which has a valley of low fitness of Hamming width $k-1$ around the global optimum. Consequently, elitist algorithms can leave this local optimum only by flipping $k$ specific bits (and~\cite{Doerr20gecco} suggests that non-elitist algorithms cannot do better). As shown in~\cite{DoerrLMN17}, for this multimodal benchmark function the insights gained previously on unimodal functions like \onemax, linear functions, or \leadingones do not apply. The optimal mutation rate for $\jump_{nk}$ was found to be $(1\pm o(1)) \frac kn$. Deviating from this optimal rate by a small constant factor leads to a runtime increase by a factor of $e^{\Omega(k)}$. Consequently, the choice of the mutation rate for this problem is truly delicate. 

To overcome this difficulty, the use of a random mutation rate chosen according to a heavy-tailed distribution, more specifically, a power-law distribution with exponent $\beta>1$, was suggested. This mutation operator, called \emph{fast mutation} in agreement with previous uses of heavy-tailed distributions in continuous evolutionary computation~\cite{SzuH87,YaoL97,YaoLL99}, samples a random number $\alpha \in [1..\lfloor \frac n2 \rfloor]$ with probability proportional to $\alpha^{-\beta}$ and then flips each bit independently with rate~$\frac \alpha n$. Each application of this operator samples the value of $\alpha$ independently. 

The main result in~\cite{DoerrLMN17} is that the \oea with this mutation operator optimizes $\jump_{nk}$ in a time that is only by a factor of $O(k^{\beta-0.5})$ larger than the time resulting from standard bit mutation with the optimal rate. Given that missing the optimal rate (which is only accessible when knowing $k$) by a small constant factor already incurs a runtime increase by a factor of $e^{\Omega(k)}$, the $O(k^{\beta-0.5})$ price for having a one-size-fits-all mutation operator appears to be a good investment. From the asymptotic point of view $\beta$ should be taken arbitrarily close to $1$, but the experiments conducted in~\cite{DoerrLMN17} suggested that $\beta = 1.5$ is a good choice. Both theory and experiments showed that the choice of $\beta$ is not overly critical. For this reason, it is fair to call fast mutation a parameterless operator.

Since the fast mutation operator is nothing else than a random linear combination of standard bit mutation operators with rates~$\frac \alpha n, \alpha = 1, \dots, \lfloor \frac n2 \rfloor$, it is not surprising that the resulting runtime is higher than the one from the best of these individual operators. Rather, it is surprising that by simply averaging over the available options, one comes relatively close to the optimum, and this in a scenario where for static rates a small deviation from the optimum leads to a significantly increased runtime.

In this work, we observe an even more surprising strength of the fast mutation operator. We investigate how the $(1+(\lambda,\lambda))$ genetic algorithm (\ollea), first proposed in Doerr, Doerr, and Ebel~\cite{DoerrDE15}, performs with the fast mutation operator. The \ollea is an evolutionary algorithm that creates $\lambda$ offspring from a unique parent individual with an unusually high mutation rate (independently, apart from the fact that they all have the same Hamming distance from the parent), selects the best of these, and creates another $\lambda$ individuals via a biased crossover between this mutation winner and the original parent. The best of these is taken as the new parent individual if it is at least as good as the previous parent (see Section~\ref{sec:preliminaries} for more details). 

This combination of a high mutation rate and crossover with the parent as repair mechanism allows the algorithm to more efficiently explore the search space when the parameters are chosen suitably. Both from informal considerations and from existing runtime results, the right parameterization seems to be that the mutation rate is $p = \frac \lambda n$ and the crossover bias, that is, the rate with which the crossover offspring takes bits from the mutation winner, is $c = \frac 1 \lambda$. The informal argument for this is that a single application of mutation and crossover generates a bit string distributed as if generated via standard bit mutation with rate $\frac 1n$. 

With a number of runtime analyses~\cite{DoerrDE15,BuzdalovD17,DoerrD18,AntipovDK19foga} supporting this choice\footnote{We note that the work~\cite{AntipovDK20} conducted in parallel to ours suggests that a different choice is necessary when large fitness valleys need to be crossed.}, we fix this relation of the three parameters in the remainder of this work. Since the mutation rate is the starting point of our research, we can alternatively first choose a mutation rate of type $p = \frac \alpha n$ and then set $\lambda = pn$ and $c = \frac{1}{pn}$. 

The right choice of the mutation rate is non-trivial. The good news from~\cite{DoerrDE15} is that any rate between $p = \omega(\frac 1n)$ and $p = o(\frac{\log n}{n})$ leads to a runtime of $o(n \log n)$ on \onemax, that is, asymptotically faster than the performance of classic evolutionary algorithms. The optimal mutation rate of \[p = \Theta\left(\frac 1n \sqrt{\frac{\log(n) \log\log(n)}{\log\log\log(n)}}\,\right),\]
however, is non-trivial to find~\cite{DoerrD18}. It yields an expected runtime on \onemax of
\[E[T] = \Theta\left(n \sqrt{\frac{\log(n) \log\log\log(n)}{\log\log(n)}}\,\right).\]

Our main research goal in this work is understanding how the \ollea performs when instead of standard bit mutation with a fixed mutation rate~$p$ the fast mutation operator is used. With the previously suggested relations between mutation rate, offspring number, and crossover bias, this means that first a number $\alpha$ is sampled from a power-law distribution, then $\lambda = \alpha$ offspring are generated via flipping $\ell$ bits chosen uniformly at random, where $\ell \sim \Bin(n, \frac{\alpha}{n})$,\footnote{This mutation can be interpreted as a standard bit mutation with rate $\frac{\alpha}{n}$, but conditional on having the same number of flipped bits for all individuals.} and finally $\lambda$ times a biased crossover with bias $c = \frac 1 \alpha$ between parent and mutation winner is performed. We call this modified algorithm the \emph{fast \ollea}.

\textbf{Our main result} is that not only the use of the fast mutation operator in the \ollea relieves us from finding a good mutation rate, but surprisingly we can even obtain a runtime that is faster than the runtime of the \ollea with any fixed mutation rate: If the power-law exponent $\beta$ satisfies $2 < \beta < 3$, then the fast \ollea has an expected runtime of $O(n)$ on \onemax. 

We note that a linear runtime of the \ollea on \onemax was obtained earlier with a self-adjusting choice of the mutation rate based on the one-fifth rule~\cite{DoerrD18}. While this worked well on \onemax,  experimental~\cite{GoldmanP14} and theoretical~\cite{BuzdalovD17} studies on satisfiable \maxsat instances showed that this approach carries the risk that the population size $\lambda$ increases rapidly because the problem structure may just not allow a one-fifth success rate, regardless how large $\lambda$ is. Since this behavior increases the time complexity of each iteration, it leads to a significant performance loss. Such problems, naturally, cannot arise with the static behavior of the fast mutation operator. 

Via an empirical study, we show that the fast mutation operator indeed without any modification also solves well the satisfiable \maxsat instances for which the one-fifth rule variant of the \ollea did not perform well  in~\cite{BuzdalovD17} (unless enriched with a suitable cap on $\lambda$). However, our study also shows that on \onemax itself, the self-adjusting \ollea is by a constant factor faster than the fast \ollea. Since the runtime loss from a degenerate behavior of the one-fifth rule version of the \ollea can be large (due to the population size of order $n$), we draw from these results the recommendation to use the more robust fast \ollea on a novel problem rather than the self-adjusting \ollea.

\section{Notation and Problem Statement}
\label{sec:preliminaries}

The \ollea, first presented in~\cite{DoerrDE15}, has the following working principles. It stores one current individual $x$, which is initialized with a random bit string. Each iteration of the \ollea consists of two phases, which are the \emph{mutation} phase and the \emph{crossover} phase. In the mutation phase the algorithm first chooses the \emph{mutation strength} $\ell$ following the binomial distribution with parameters $n$ and $p$, where $p$ is usually called the \emph{mutation rate}. It then creates $\lambda$ mutants by copying the current individual $x$ and flipping exactly $\ell$ bits which are chosen uniformly at random, independently for each mutant. After that the mutant with the best fitness is chosen as the winner of the mutation phase $x'$ (all ties are broken uniformly at random). In the crossover phase the algorithm $\lambda$ times performs a crossover between $x$ and $x'$ by taking each bit from $x'$ with probability $c$ and from $x$ otherwise. The probability $c$ is called the \emph{crossover bias.} The best crossover offspring $y$ (all ties are again broken uniformly at random) is compared with the current individual $x$. If $y$ is not worse, then it replaces $x$. The main hope behind this algorithm is that with a high mutation rate, the mutation winner $x'$ contains some beneficial solution elements, and that the crossover with the parent acts as repair mechanism that removes the destructions caused by the high mutation rate.

As it was discussed in the introduction, the standard parameter setting proposed in~\cite{DoerrDE15} uses the mutation rate $p = \frac{\lambda}{n}$ and the crossover bias $c = \frac{1}{\lambda}$. However, there is not strong recommendation on how to choose $\lambda$. For the static choice~\cite{DoerrDE15} suggests to use $\lambda = \omega(1)$ and $\lambda = o(\log(n))$ in order to have a $o(n\log n)$ runtime on \onemax, but this runtime is still super-linear. It was also shown in~\cite{DoerrDE15} that choosing a fitness-dependent $\lambda = \sqrt{\frac{n}{n - f(x)}}$ gives a linear runtime on \onemax. In~\cite{DoerrD18} it was shown that if we control $\lambda$ according to the one-fifth rule we also get a $\Theta(n)$ runtime on \onemax.




In this paper we propose to choose $\lambda$ in each iteration from some heavy-tailed distribution. More precisely, the probability that we choose $\lambda = i$ is
\begin{align*}
	\Pr[\lambda = i] = \begin{cases}
		C_{\beta, u} i^{-\beta}, &\text{ if } i \in [1..u], \\
		0, &\text{ otherwise,}\\
	\end{cases}
\end{align*}
where $\beta \in \R$ is the power-law exponent of the distribution (which is always considered as a constant), $u \in \N$ is an upper bound on the choice of $\lambda$ (and may depend on $n$), and $C_{\beta, u} \coloneqq (\sum_{i = 1}^u i^{-\beta})^{-1}$ is the normalization coefficient. All our runtime results on \onemax will hold for the classic choice $u = \lfloor n/2 \rfloor$. We introduce this additional parameter because the Max-SAT analyses in~\cite{BuzdalovD17} showed that sometimes a stricter upper bound on $\lambda$ is necessary. For that reason, it is interesting to see also in the \onemax analyses how small an upper bound on $\lambda$ can be taken so that a linear runtime is still obtained. 

The detailed pseudocode of the fast \ollea is shown in Algorithm~\ref{alg:ollea}. Our main result will be that this simple way of choosing $\lambda$ gives us a linear runtime for all $\beta \in (2, 3)$ and $u \ge \ln^{\frac{1}{3 - \beta}}(n)$.

\begin{algorithm}[h]
	$x \gets $ random bit string of length $n$\; 
    \While{not terminated}
        {
		Choose $\lambda$ from $[1..u]$ with $\Pr[\lambda = i] \sim i^{-\beta}$\;
        Choose $\ell \sim \Bin\left(n, \frac{\lambda}{n}\right)$\;
        \For{$i \in [1..\lambda]$}
            {$x^{(i)} \gets$ a copy of $x$\;
            Flip $\ell$ bits in $x^{(i)}$ chosen uniformly at random\;
            }
        $x' \gets \argmax_{z \in \{x^{(1)}, \dots, x^{(\lambda)}\}}f(z)$\;
        \For{$i \in [1..\lambda]$}
            {Create $y^{(i)}$ by taking each bit from $x'$ with probability $\frac{1}{\lambda}$ and from $x$ with probability $\frac{\lambda - 1}{\lambda}$\;
            }
        $y \gets \argmax_{z \in \{y^{(1)}, \dots, y^{(\lambda)}\} }f(z)$\;
        \If{$f(y) \ge f(x)$}
            {
             $x \gets y$\;   
            }
        }
	\caption{The fast \ollea with power-law exponent $\beta$ and upper limit $u$ maximizing $f:\{0,1\}^n\to\R$}
	\label{alg:ollea}
\end{algorithm}

\subsection{Useful Tools}
\label{sec:useful-tools}

In this section we collect some classic results which are used in our proofs. First, to be able to make the transition between the number of iterations and the number of fitness evaluations, we use Wald's equation~\cite{Wald45}.

\begin{lemma}[Wald's equation]\label{lem:wald}
	Let $(X_t)_{t \in \N}$ be a sequence of real-valued random variables and let $T$ be a positive integer random variable. Let also all following conditions be true.
	\begin{enumerate}
		\item All $X_t$ have the same finite expectation.
		\item For all $t \in \N$ we have $E[X_t \mathds{1}_{\{T \ge t\}}] = E[X_t] \Pr[T \ge t]$.
		\item $\sum_{t = 1}^{+\infty} E[|X_t| \mathds{1}_{\{T \ge t\}}] < \infty$.
		\item $E[T]$ is finite.
	\end{enumerate}
	Then we have
	\[
		E\left[\sum_{t = 1}^{T} X_t\right] = E[T]E[X_1].	
	\]
\end{lemma}

We use the following inequality to estimate the probability that at least one of $\lambda$ Bernoulli trials succeeds. 

\begin{lemma}
	\label{lem:Bernoulli}
	For all $p \in [0, 1]$ and all $\lambda > 0$ we have
	\[
		1 - (1 - p)^\lambda \ge \frac{\lambda p}{1 + \lambda p}.	
	\]
\end{lemma}
\begin{proof}
	By~\cite[Lemma 8]{RoweS14} (or~(1.4.19) in~\cite{Doerr20bookchapter}) we have $(1 - p)^\lambda \le \frac{1}{1 + \lambda p}$. Hence,
	\begin{align*}
        1 - \left(1 - p\right)^\lambda &\ge 1 - \frac{1}{1 + \lambda p} = \frac{\lambda p}{1 + \lambda p}. \qedhere
	\end{align*}
\end{proof}

We frequently use the following bounds on the partial sums of the generalized harmonic series. 

\begin{lemma}\label{lem:sum-lower}
	For all $s \in \R$ such that $s \ge 1$ and for all $\alpha \ne 1$ we have $\sum_{i = 1}^{\lceil s \rceil} i^{-\alpha} \ge \frac{s^{1 - \alpha} - 1}{1 - \alpha}$. For $\alpha = 1$ we have $\sum_{i = 1}^{\lceil s \rceil} i^{-\alpha} \ge \ln(s)$.
\end{lemma}
\begin{proof}
	We estimate the sum for $\alpha \ne 1$ through the corresponding integral (this estimate is illustrated in Figures~\ref{fig:integral-negative} and~\ref{fig:integral-positive}).
	\begin{align*}
		\sum_{i = 1}^{\lceil s \rceil} i^{-\alpha} \ge \int_{1}^s x^{-\alpha} dx = \frac{s^{1 - \alpha} - 1}{1 - \alpha}.
	\end{align*}

	The case for $\alpha = 1$ is a well-known bound on the partial sum of the harmonic series. 
\end{proof}

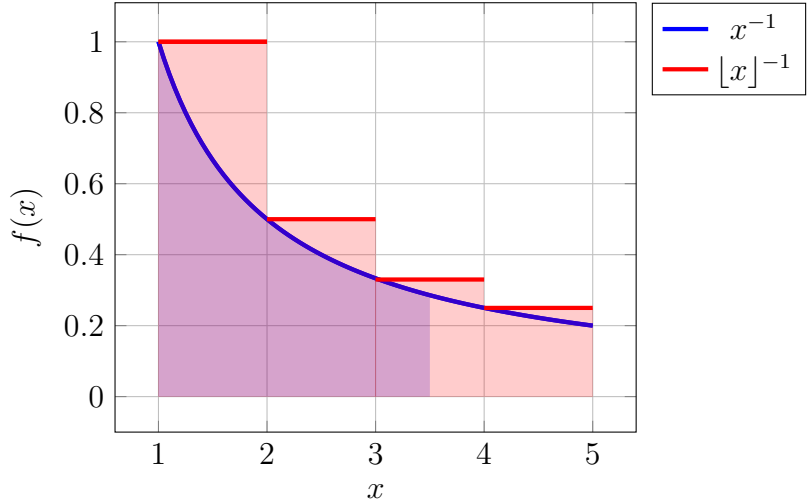
\begin{figure}
    \caption{Illustration of the inequality $\sum_{i = 1}^{\lceil s \rceil} i^{-\alpha}\ge \int_1^s x^{-\alpha} dx$ for the case $\alpha \ge 0$. In this example we have $\alpha = 1$, $s = 3.5$ and $\lceil s \rceil = 4$. The red area equals to the sum. The blue area (fully under red, thus purple) equals to the integral.}
	\label{fig:integral-negative}
    \begin{center}
        \begin{tikzpicture}
          \begin{axis}[
          grid=major,
          ylabel={$f(x)$},
          xlabel={$x$},
          legend pos=outer north east,
          ymin=-0.1]
    
          \addplot[domain=1:5, samples=81, draw=blue, ultra thick]{f1(x)};
          \addlegendentry{$x^{-1}$}
          \addplot[draw=red, ultra thick] coordinates
          {(1,1) (2, 1)};
          \addlegendentry{$\lfloor x \rfloor^{-1}$}
    
            \addplot[domain=1:3.5, samples=81, draw=none, fill=blue, fill opacity=0.2]{f1(x)} \closedcycle;
            \addplot[domain=1:5, samples=81, draw=blue, ultra thick]{f1(x)};
    
            \addplot[draw=none, fill=red, fill opacity=0.2] coordinates
            {(1,1) (2, 1)} \closedcycle;
            \addplot[draw=none, fill=red, fill opacity=0.2] coordinates
            {(2,0.5) (3, 0.5)} \closedcycle;
            \addplot[draw=none, fill=red, fill opacity=0.2] coordinates
            {(3,0.33) (4, 0.33)} \closedcycle;
            \addplot[draw=none, fill=red, fill opacity=0.2] coordinates
            {(4,0.25) (5, 0.25)} \closedcycle;
        
            \addplot[draw=red, ultra thick] coordinates
            {(1,1) (2, 1)};
            \addplot[draw=red, ultra thick] coordinates
            {(2,0.5) (3, 0.5)};
            \addplot[draw=red, ultra thick] coordinates
            {(3,0.33) (4, 0.33)};
            \addplot[draw=red, ultra thick] coordinates
            {(4,0.25) (5, 0.25)};
          \end{axis}
      \end{tikzpicture}
    \end{center}
\end{figure}

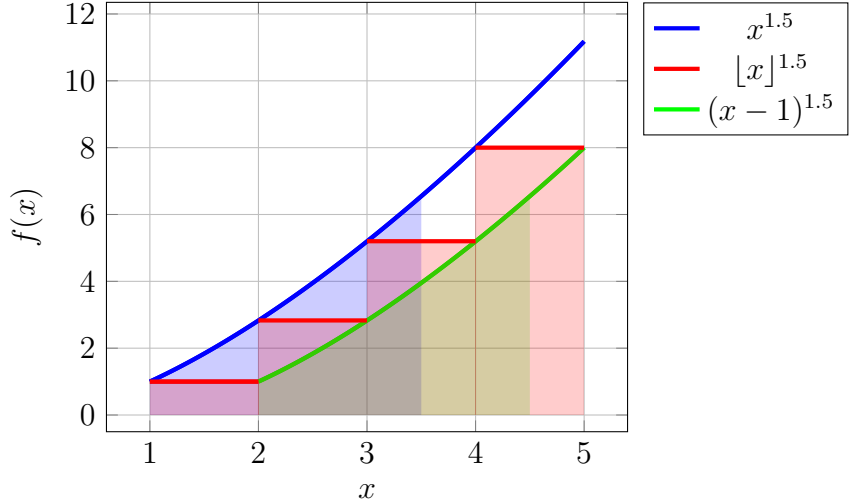
\begin{figure}    
    \caption{Illustration of the inequality $\sum_{i = 1}^{\lceil s \rceil} i^{-\alpha}\ge \int_1^s x^{-\alpha} dx$ for the case $\alpha < 0$. In this example we have $\alpha = -1.5$, $s = 3.5$ and $\lceil s \rceil = 4$. The red area equals to the sum. The blue area equals to the integral and to the green area, which is fully under the red one.}
	\label{fig:integral-positive}
    \begin{center}
      \begin{tikzpicture}
        \begin{axis}[
        grid=major,
        ylabel={$f(x)$},
        xlabel={$x$},
        legend pos=outer north east,
        ymin=-0.5]
    
        \addplot[domain=1:5, samples=81, draw=blue, ultra thick]{f2(x)};
        \addlegendentry{$x^{1.5}$}
        \addplot[draw=red, ultra thick] coordinates
        {(1,1) (2, 1)};
        \addlegendentry{$\lfloor x \rfloor^{1.5}$}
    
        \addplot[domain=2:5, samples=81, draw=green, ultra thick]{f3(x)};
        \addlegendentry{$(x - 1)^{1.5}$}

          \addplot[domain=1:3.5, samples=81, draw=none, fill=blue, fill opacity=0.2]{f2(x)} \closedcycle;
          \addplot[domain=1:5, samples=81, draw=blue, ultra thick]{f2(x)};
    
          \addplot[domain=2:4.5, samples=81, draw=none, fill=green, fill opacity=0.2]{f3(x)} \closedcycle;
          \addplot[domain=2:5, samples=81, draw=green, ultra thick]{f3(x)};
    
          \addplot[draw=none, fill=red, fill opacity=0.2] coordinates
          {(1,1) (2, 1)} \closedcycle;
          \addplot[draw=none, fill=red, fill opacity=0.2] coordinates
          {(2,2.83) (3, 2.83)} \closedcycle;
          \addplot[draw=none, fill=red, fill opacity=0.2] coordinates
          {(3,5.2) (4, 5.2)} \closedcycle;
          \addplot[draw=none, fill=red, fill opacity=0.2] coordinates
          {(4,8) (5, 8)} \closedcycle;
      
          \addplot[draw=red, ultra thick] coordinates
          {(1,1) (2, 1)};
          \addplot[draw=red, ultra thick] coordinates
          {(2,2.83) (3, 2.83)};
          \addplot[draw=red, ultra thick] coordinates
          {(3,5.2) (4, 5.2)};
          \addplot[draw=red, ultra thick] coordinates
          {(4,8) (5, 8)};
        \end{axis}
    \end{tikzpicture}
    \end{center}
\end{figure}

\begin{lemma}\label{lem:sum-upper}
	For all $u \in \N$ we have 
	\begin{itemize}
		\item $\sum_{i = 1}^{u} i^{-\alpha} \le u^{1 - \alpha}\frac{2 - \alpha}{1 - \alpha}$, if $\alpha < 0$,
		\item $\sum_{i = 1}^{u} i^{-\alpha} \le \frac{u^{1 - \alpha}}{1 - \alpha}$, if $\alpha \in [0, 1)$,
		\item $\sum_{i = 1}^{u} i^{-\alpha} \le \frac{\alpha}{\alpha - 1}$, if $\alpha > 1$,
		\item $\sum_{i = 1}^{u} i^{-\alpha} \le \ln(u) + 1$, if $\alpha = 1$.
	\end{itemize}
\end{lemma}

\begin{proof}[Proof of Lemma~\ref{lem:sum-upper}]
	By analogy with Lemma~\ref{lem:sum-lower} we estimate the sum through a corresponding integral. 
	If $\alpha < 0$ we have
	\begin{align*}
		\sum_{i = 1}^{u} i^{-\alpha} &\le \int_{1}^{u} x^{-\alpha} dx  + u^{-\alpha} \le \frac{u^{1 - \alpha} - 1}{1 - \alpha} + u^{-\alpha} \le u^{1 - \alpha}\frac{2 - \alpha}{1 - \alpha}. \\ 
	\end{align*}
	If $\alpha \ge 0$ we have
	\begin{align*}
		\sum_{i = 1}^{u} i^{-\alpha} &\le 1 + \int_{2}^{u + 1} (x - 1)^{-\alpha} dx \le 1 + \frac{u^{1 - \alpha} - 1}{1 - \alpha} \\ 
	\end{align*}
	If $\alpha \in [0, 1)$, then we have
	\begin{align*}
		\sum_{i = 1}^{u} i^{-\alpha} &\le \frac{u^{1 - \alpha} - 1 + 1 - \alpha}{1 - \alpha} \le \frac{u^{1 - \alpha}}{1 - \alpha}. \\
	\end{align*}
	If $\alpha > 1$, we have
	\begin{align*}
		\sum_{i = 1}^{u} i^{-\alpha} &\le 1 + \frac{1}{\alpha - 1} \le \frac{\alpha}{\alpha - 1}. \\
	\end{align*}

	The case for $\alpha = 1$ is a well-known bound on the partial sum of the harmonic series. 
\end{proof}

\section{Runtime Analysis}
\label{sec:analysis}

In this section we prove upper and lower bounds on the runtime of the fast \ollea on \onemax.

\subsection{Upper Bound}
\label{sec:upper-bound}

Our aim in this subsection is to prove an upper bound on the number of fitness evaluations taken until the fast \ollea finds the optimum of the \onemax benchmark. Since it is technically easier, we first regard the number of iterations until the optimum is found. For algorithms with fixed population sizes, such a bound would immediately imply a bound on the number of fitness evaluations (namely by multiplying the number of iterations with the fixed number of fitness evaluations per iteration). For the fast \ollea using a newly sampled value of $\lambda$ in each iteration, things are not that easy, but Wald's equation (Lemma~\ref{lem:wald}) allows to argue that multiplying with the expected number of fitness evaluations per iteration gives the right result.

Before proceeding with proofs, we now state two theorems that together constitute the main result of this subsection.
We start by showing that for reasonable parameter values, the optimum is found in a linear number of iterations.

\begin{theorem}\label{thm:iterations}
	If $\beta \in (1, 3)$ and $u \ge \ln^\frac{1}{3 - \beta}(n)$, then the expected number of iterations until the fast \ollea finds the optimum of \onemax function is $O(n)$.
\end{theorem}

When $\beta > 2$, the expected number of fitness evaluations per iteration is $\Theta(1)$ (see Lemma~\ref{lem:exp-lambda}). With this observation and Wald's equation, we obtain the following estimate for the runtime.

\begin{theorem}\label{thm:fitness-evaluations}
	If $\beta \in (2, 3)$ and $u \ge \ln^\frac{1}{3 - \beta}(n)$, then the expected number of fitness evaluations until the fast \ollea finds the optimum of \onemax function is $O(n)$.
\end{theorem}

We start with the proof of Theorem~\ref{thm:iterations}. For the readers' convenience we split the proof into Lemmas~\ref{lem:progress-lambda} and~\ref{lem:progress-global}. The first lemma is essentially an interpretation of Lemma~7 in~\cite{DoerrDE15}.

\begin{lemma}\label{lem:progress-lambda}
	If $\lambda \le \sqrt{\frac{n}{d(x)}}$, where $d(x)$ is the current Hamming distance between the current individual $x$ and the optimum, then the probability $p_{d(x)}(\lambda)$ of increasing the fitness in one iteration is at least
	\[
		C\frac{d(x)\lambda^2}{n},	
	\] 
	where $C > 0$ is an absolute constant. If $\lambda > \sqrt{\frac{n}{d(x)}}$, then this probability is at least $C$.
\end{lemma}
\begin{proof}
	By~\cite[Lemma~7]{DoerrDE15}, the probability of a true progress (that is, an iteration in which we find a strictly better individual than the current individual $x$) $p_{d(x)}(\lambda)$ is at least
	\[
		C'\left(1 - \left(1 - \frac {d(x)}n\right)^\frac{\lambda^2}{2}\right),	
	\]
	where $C' > 0$ is an absolute constant. By Lemma~\ref{lem:Bernoulli} we have

	\begin{align*}
		p_{d(x)}(\lambda) &\ge C'\left(1 - \left(1 - \frac {d(x)}n\right)^\frac{\lambda^2}{2}\right) 
						  \ge C'\frac{\frac{d(x)\lambda^2}{2n}}{1 + \frac{d(x)\lambda^2}{2n}}. \\
	\end{align*}
	If $\lambda \le \sqrt{\frac{n}{d(x)}}$, then we have $p_{d(x)}(\lambda) \ge C' \frac{d(x)\lambda^2}{3n}.$
	Note that $C \coloneqq \frac{C'}{3}$ is an absolute constant as well as $C'$. If $\lambda > \sqrt{\frac{n}{d(x)}}$, then $p_{d(x)}(\lambda) \ge \frac{C'}{3} = C.$

\end{proof}

\begin{lemma}\label{lem:progress-global}
	Let $\beta \in (1, 3)$. Then the probability $p_{d(x)}$ of having progress in one iteration given that the current distance to the optimum is $d(x)$ is at least 
	\[
		C(\beta)\frac{d(x)U^{3 - \beta}}{n},
	\]
	where $U = \min\{u, \sqrt{\frac{n}{d(x)}}\}$ and $C(\beta)$ is some constant independent of~$n$.
\end{lemma}

\begin{proof}
	Note that since $u$ is an integer number, we have $u \ge \lceil U \rceil$. Hence, by Lemma~\ref{lem:progress-lambda} we have
	\begin{align*} 
		p_{d(x)} =   \sum_{\lambda = 1}^u C_{\beta, u} \lambda^{-\beta} p_{d(x)}(\lambda) 
				 \ge C_{\beta, u} C \sum_{\lambda = 1}^{\lceil U \rceil}  \frac{d(x)\lambda^{2 - \beta}}{n}
				 =   C_{\beta, u} C \frac{d(x)}{n} \sum_{\lambda = 1}^{\lceil U \rceil}  \lambda^{2 - \beta}
	\end{align*} 

	If $U \ge 2$, then by Lemma~\ref{lem:sum-lower} we have
	\begin{align*}
			\sum_{\lambda = 1}^{\lceil U \rceil}  \lambda^{2 - \beta} &\ge \frac{U^{3 - \beta} - 1}{3 - \beta} \ge \frac{1 - 2^{\beta - 3}}{3 - \beta}U^{3 - \beta} \ge \frac{3}{8}U^{3 - \beta}.
	\end{align*}
	Otherwise, when $U < 2$ we have
	\begin{align*}
			\sum_{\lambda = 1}^{\lceil U \rceil} \lambda^{2 - \beta} &\ge 1 = U^{\beta - 3} U^{3 - \beta} \ge 2^{\beta - 3}U^{3 - \beta} \ge \frac{1}{4}U^{3 - \beta}.
	\end{align*}

	Finally, we estimate
	\begin{align*}
		p_{d(x)} \ge C_{\beta, u} C \frac{d(x)}{n} \sum_{\lambda = 1}^{\lceil U \rceil}  \lambda^{2 - \beta}
  		         \ge C_{\beta, u} C \frac{1}{4}\frac{d(x)}{n}U^{3 - \beta} = C(\beta) \frac{d(x)U^{3 - \beta} }{n}
	\end{align*}
	with $C(\beta) \coloneqq \frac{1}{4}C_{\beta, u} C$. Since $C$ is an absolute constant by Lemma~\ref{lem:progress-lambda} and since, by Lemma~\ref{lem:sum-upper}, $C_{\beta, u}$ is at least $\frac{\beta - 1}{\beta}$, which is a constant independent of $u$, $C(\beta)$ is also a constant independent of $u$. 
\end{proof}

In order to show a full picture we also computed the values of $p_{d(x)}$ for a wider range of parameters $u$ and $\beta$. The results are shown in Table~\ref{tbl:progress} and their proofs are included in the appendix. 

\begin{table}
	\caption{The probability $p_{d(x)}$ to increase fitness in one iteration for various values of parameters $\beta \in \R$ and $u \in \N$.}
	\label{tbl:progress}
	\centering
	\begin{tabular}{|c||l|l|}
		\hline
		$\beta$                        & $u \le \sqrt{\frac{n}{d(x)}}$                    & $u > \sqrt{\frac{n}{d(x)}}$  \\ \hline
		$<1$ 	                       & $\Omega\left(\frac{d(x)u^2}{n}\right)$           & $\Omega(1)$ \\ \hline
		$=1$	                       & $\Omega\left(\frac{d(x)u^2}{n\log(u)}\right)$    & $\ge \frac{1 + \ln(u) - \ln(\sqrt{\frac{n}{d(x)}})}{36\ln(u)}$ \\ \hline
		$(1, 3)$ & $\Omega\left(\frac{d(x)u^{3 - \beta}}{n}\right)$ & $\Omega\left(\sqrt{\frac{n}{d(x)}}^{1 - \beta}\right)$ \\ \hline
		$=3$ 	                       & $\Omega\left(\frac{d(x)\log(u)}{n}\right)$       & $\Omega\left(\frac{\log(n/d(x)) + 1}{n/d(x)}\right)$ \\ \hline	
		$>3$ 	                       & \multicolumn{2}{c|}{$\Omega\left(\frac{d(x)}{n}\right)$} \\ \hline	
    \end{tabular}
\end{table}




We are now ready to prove Theorem~\ref{thm:iterations}.

\begin{proof}[Proof of Theorem~\ref{thm:iterations}]
	We estimate the upper bound on the expectation of the runtime $T_I$ (in terms of iterations) as the sum of expected times until the algorithm leaves each fitness level. By Lemma~\ref{lem:progress-global} we have 
	\begin{align*}
		E[T_I] \le \sum_{d = 1}^n \frac{1}{p_{d}}
		       \le \frac{1}{C(\beta)} \left(\sum_{d = 1}^{\lfloor n/u^2 \rfloor} \frac{n}{du^{3 - \beta}} + \sum_{d = \lfloor n/u^2 \rfloor + 1}^n \sqrt{\frac{n}{d}}^{\beta - 1}\right).
	\end{align*}
	By Lemma~\ref{lem:sum-upper} we estimate the first sum
	\begin{align*}
		\sum_{d = 1}^{\lfloor n/u^2 \rfloor} \frac{n}{du^{3 - \beta}} &\le \frac{n\left(\ln\left(\frac{n}{u^2}\right) + 1\right)}{u^{3 - \beta}} \le \frac{n(\ln(n) + 1)}{\ln(n)} = n(1 + o(1)),
	\end{align*}
	where in the last inequality we used the assumption $u \ge \ln^{\frac{1}{3 - \beta}}(n)$. By Lemma~\ref{lem:sum-upper} we estimate the second sum as follows.

	\begin{align*}
		\sum_{d = \lfloor n/u^2 \rfloor + 1}^n &\sqrt{\frac{n}{d}}^{\beta - 1} \le \sum_{d = 1}^n \sqrt{\frac{n}{d}}^{\beta - 1} \\
																					&\le n^\frac{\beta - 1}{2} \sum_{d = 1}^n d^{-\frac{\beta - 1}{2}} 
																					\le n^\frac{\beta - 1}{2}  \frac{n^\frac{3 - \beta}{2}}{(3 - \beta)/2} = O(n).
	\end{align*}

	Therefore, we have
	\begin{align*}
		E[T_I] &\le \frac{1}{C(\beta)}\left(O(n) + O(n)\right) = O(n). \qedhere
	\end{align*}
\end{proof}

Before we prove Theorem~\ref{thm:fitness-evaluations} we first estimate $E[\lambda]$, which is half the expected cost of one iteration.

\begin{lemma}\label{lem:exp-lambda}
	If $\lambda$ is sampled from the heavy-tailed distribution with parameter $\beta$ and upper limit $u$, then its expected value is
	\begin{itemize}
		\item $E[\lambda] = \Theta(1)$, if $\beta > 2$,
		\item $E[\lambda] = \Theta(\log(u))$, if $\beta = 2$,
		\item $E[\lambda] = \Theta(u^{2 - \beta})$, if $\beta \in (1, 2)$,
		\item $E[\lambda] = \Theta(\frac{u}{\log(u)})$, if $\beta = 1$, and
		\item $E[\lambda] = \Theta(u)$, if $\beta < 1$,
	\end{itemize}
	where the asymptotic notation is for $u \to +\infty$.
\end{lemma}
\begin{proof}
	First recall that $C_{\beta, u} = (\sum_{i = 1}^u i^{-\beta})^{-1}$. By Lemmas~\ref{lem:sum-lower} and~\ref{lem:sum-upper} we have 
	\begin{itemize}
		\item if $\beta < 1$, then $C_{\beta, u} = \Theta(u^{\beta - 1})$,
		\item if $\beta = 1$, then $C_{\beta, u} = \Theta(1/\ln(u))$, and
		\item if $\beta > 1$, then $C_{\beta, u} = \Theta(1)$.
	\end{itemize} 
	We compute
	\begin{align*}
		E[\lambda] = \sum_{i = 1}^u i\Pr[\lambda = i] = C_{\beta, u} \sum_{i = 1}^u i^{1 - \beta}.  
	\end{align*}

	If $\beta > 2$, then by Lemma~\ref{lem:sum-upper} we have
	\begin{align*}
		C_{\beta, u} \le E[\lambda] \le C_{\beta, u}\frac{\beta - 1}{\beta - 2}.
	\end{align*}
	Hence, $E[\lambda] = \Theta(1)$.

	If $\beta = 2$, then $\sum_{i = 1}^u i^{1 - \beta}$ is a partial sum of the harmonic series, thus it is $\Theta(\log(u))$.
	If $\beta < 2$, then by Lemmas~\ref{lem:sum-lower} and~\ref{lem:sum-upper} we have
	\begin{align*}
		C_{\beta, u} \frac{u^{2 - \beta} - 1}{2 - \beta} \le E[\lambda] \le C_{\beta, u} \frac{u^{2 - \beta}}{2 - \beta}.
	\end{align*}
	Therefore, $E[\lambda] = C_{\beta, u} \Theta(u^{2 - \beta})$. Together with the estimates of $C_{\beta, u}$ this proves the lemma for $\beta < 2$.
\end{proof}

We are now in the position to prove Theorem~\ref{thm:fitness-evaluations}

\begin{proof}[Proof of Theorem~\ref{thm:fitness-evaluations}]
	Let $\{\lambda_t\}_{t\in\N}$ be a sequence of random variables, each following the power-law distribution with parameters $\beta$ and $u$. We can assume that for all $t \in \N$ the fast \ollea chooses $\lambda \coloneqq \lambda_t$ in iteration $t$. Since the cost of one iteration is $2\lambda$ fitness evaluations ($\lambda$ for the mutation phase and $\lambda$ for the crossover phase), the total number of fitness evaluations $T_F$ has the same distribution as
	$
		\sum_{t = 1}^{T_I} 2\lambda_t.	
	$
	We aim at proving that the sequence $(\lambda_t)_{t \in \N}$ and $T_I$ allow to use Wald's equation (Lemma~\ref{lem:wald}). We show that conditions $(1)$--$(4)$ of this lemma are satisfied.

	\begin{enumerate}
		\item All $\lambda_t$ have the same expectation, which is finite by Lemma~\ref{lem:exp-lambda}.
		\item The event $T_I \ge t$ is independent of the outcome of $\lambda_t$, which implies that for all $i \in [1..u]$ we have $\Pr[T_I \ge t \mid \lambda_t = i] = \Pr[T_I \ge t]$. Therefore, we have
		\begin{align*}
			E[\lambda_t\mathds{1}_{\{T_I \ge t\}}] &= \sum_{i = 1}^{u} i \Pr[\lambda_t = i] \Pr[T_I \ge t \mid \lambda_t = i] \\
												   &= \Pr[T_I \ge t] \sum_{i = 1}^{u} i \Pr[\lambda_t = i] = \Pr[T_I \ge t] E[\lambda_t].
		\end{align*}
		\item By the previous condition we have
		\begin{align*}
			\sum_{t = 1}^{+\infty} E[|\lambda_t| \cdot \mathds{1}_{\{T_I \ge t\}}] &= \sum_{t = 1}^{+\infty} \Pr[T_I \ge t] E[\lambda_t] = E[\lambda] E[T_I],
		\end{align*}
		since for all $t \in \N$ we have $E[\lambda_t] = E[\lambda]$. By Theorem~\ref{thm:iterations} and Lemma~\ref{lem:exp-lambda}, both $E[\lambda]$ and $E[T_I]$ are finite, hence their product is finite as well.
		\item By Theorem~\ref{thm:iterations} $E[T_I]$ is finite.
	\end{enumerate}
	Thus, by Wald's inequality we have 
	\begin{align*}
		E[T_F] = E[T_I]E[2 \lambda_t].
	\end{align*}
	By Theorem~\ref{thm:iterations} and Lemma~\ref{lem:exp-lambda} we conclude
	\begin{align*}
		E[T_F] = O(n) \cdot \Theta(1) = O(n).   
		& \qedhere
	\end{align*}

\end{proof}

Although we are mostly interested in $\beta  \in (2, 3)$ and reasonably high upper limit $u$, a reader might find it interesting to see the upper bounds for the runtimes yielded by different parameters values.

For this reason we show the estimates for $E[T_I]$ and $E[T_F]$ for a wider range of parameters values in Table~\ref{tbl:runtimes} and their proofs are included in the appendix. 

\begin{table}
	\caption{Upper bounds on the expected number of iterations and expected number of fitness evaluations for different values of $\beta$ and $u$. The last column is calculated by Wald's equation in the same manner as in Theorem~\ref{thm:fitness-evaluations}.}
	\label{tbl:runtimes}
	\begin{scriptsize}
		\begin{tabular}{|c||c|c|}
			\hline
		    $\beta$ & $E[T_I]$ &  $E[T_F] = 2E[T_I]E[\lambda]$ \\ \hline
			$<1$ 	& \begin{tabular}{ll}
						$O(n)$                                        & if $u \ge \sqrt{\ln(n)}$ \\
						$O\left(\frac{n}{u^2}\log\frac{n}{u^2}\right)$ & if $u \le \sqrt{\ln(n)}$ 
					  \end{tabular}  & 
					  \begin{tabular}{ll}
						$O(nu)$                                        & if $u \ge \sqrt{\ln(n)}$ \\
						$O\left(\frac{n}{u}\log\frac{n}{u^2}\right)$ & if $u \le \sqrt{\ln(n)}$ 
					  \end{tabular} \\ \hline
			$=1$	& \begin{tabular}{ll}
						$O(n)$ & if $u \ge \sqrt{\ln(n)\ln\ln(n)}$ \\
						$O\left(\frac{n}{u^2}\log\left(\frac{n}{u^2}\right)\log(u)\right)$ & if $u \le \sqrt{\ln(n)\ln\ln(n)}$ 
					  \end{tabular}	 & 
					  \begin{tabular}{ll}
						$O(\frac{nu}{\log(u)})$ & if $u \ge \sqrt{\ln(n)\ln\ln(n)}$ \\
						$O\left(\frac{n}{u}\log\left(\frac{n}{u^2}\right)\right)$ & if $u \le \sqrt{\ln(n)\ln\ln(n)}$ 
					  \end{tabular}\\ \hline
			$(1, 2)$& \multirow{3}{*}{\begin{tabular}{ll}
										$O(n)$ & if $u \ge \ln^{\frac{1}{3 - \beta}}(n)$ \\
										$O\left(\frac{n}{u^{3 - \beta}} \log\left(\frac{n}{u^2}\right)\right)$ & if $u < \ln^{\frac{1}{3 - \beta}}(n)$
					  \end{tabular}} & 
					  \begin{tabular}{ll}
						$O(nu^{2 - \beta})$ & if $u \ge \ln^{\frac{1}{3 - \beta}}(n)$ \\
						$O\left(\frac{n}{u} \log\left(\frac{n}{u^2}\right)\right)$ & if $u < \ln^{\frac{1}{3 - \beta}}(n)$
	 				  \end{tabular} \\ \cline{1-1}\cline{3-3}
			$=2$ 	& & 
					  \begin{tabular}{ll}
						$O(n\log(u))$ & if $u \ge \ln(n)$ \\
						$O\left(\frac{n\log(u)}{u} \log\left(\frac{n}{u^2}\right)\right)$ & if $u < \ln(n)$
					  \end{tabular} \\ \cline{1-1}\cline{3-3}
			$(2, 3)$& & 
					  \begin{tabular}{ll}
						$O(n)$ & if $u \ge \ln^{\frac{1}{3 - \beta}}(n)$ \\
						$O\left(\frac{n}{u^{3 - \beta}} \log\left(\frac{n}{u^2}\right)\right)$ & if $u < \ln^{\frac{1}{3 - \beta}}(n)$
					  \end{tabular} \\ \hline
			$=3$	& \begin{tabular}{ll}
						$O(n\log\log(u))$ & if $u \ge n^\frac{1}{\ln\ln(n)}$ \\
						$O\left(\frac{n}{\log(u)}\log\left(\frac{n}{u^2}\right)\right)$ & if $u < n^\frac{1}{\ln\ln(n)}$
					  \end{tabular}  &  
					  \begin{tabular}{ll}
						$O(n\log\log(u))$ & if $u \ge n^\frac{1}{\ln\ln(n)}$ \\
						$O\left(\frac{n}{\log(u)}\log\left(\frac{n}{u^2}\right)\right)$ & if $u < n^\frac{1}{\ln\ln(n)}$
					  \end{tabular} \\ \hline
			$>3$	& $O(n\log(n))$	& 	$O(n\log(n))$ \\ \hline
		\end{tabular}
	\end{scriptsize}
\end{table}

In the proofs of Theorems~\ref{thm:iterations} and~\ref{thm:fitness-evaluations} we aimed at delivering only asymptotic upper bounds disregarding the leading constant in order not to reduce the readability of the paper. However, for the complete picture, without proof we estimate the leading constant delivered by our arguments.

Recall that $C(\beta) = \frac{1}{12}C_{\beta, u} C'$. From the proof of Lemma~7 in~\cite{DoerrDE15} we can show that $C'$ which is used in Lemma~\ref{lem:progress-lambda} is at least $\frac{1}{e}(1 - \exp(-\exp(-\frac{3}{2}))) \approx 0.0735$. For any upper bound $u = \omega(1)$ we also have $C_{\beta, u} \ge \frac{\beta - 1}{\beta}$. Hence, we estimate the upper bound on the leading constant.
\begin{align*}
	\frac{1}{C(\beta)}\left(1 + \frac{2}{3 -\beta}\right) &\le \frac{12\beta(5 - \beta)}{(3 - \beta)(\beta - 1)C'} \approx 164\frac{\beta(5 - \beta)}{(3 - \beta)(\beta - 1)}.
\end{align*}

Taking into account the leading constant hidden in Lemma~\ref{lem:exp-lambda}, which is $\frac{\beta - 1}{\beta - 2}$ if $\beta > 2$, we estimate the upper bound on the leading constant for $E[T_F]$ delivered by Theorem~\ref{thm:fitness-evaluations} as
\begin{equation}
	328 \cdot \frac{\beta(5 - \beta)}{(3 - \beta)(\beta - 2)}. \label{leading-constant-estimation}
\end{equation}

\subsection{Lower Bound}
\label{sec:lower-bound}

In this section we prove the tightness of our upper bounds by showing a lower bound of $\Omega(n)$ fitness evaluations for the runtime of the fast \ollea on \onemax. This is a special case of a deeper result~\cite{TeytaudG06}, which showed the same lower bound for all comparison-based algorithms (which the \ollea is). For the readers' convenience, we give an elementary proof as well.



\begin{theorem}
	\label{thm:lower-bound}
	The expected runtime of the fast \ollea with parameter $\beta \in \R$ and any upper limit $u \in \N$ on the \onemax function is at least $\Omega(\frac{n}{E[\lambda]})$ iterations, where $E[\lambda]$ is estimated as in Lemma~\ref{lem:exp-lambda}, and $\Omega(n)$ fitness evaluations.
\end{theorem}
\begin{proof}
	The progress in one iteration cannot be greater than the number $\ell$ of bits which we flip in each mutant, since we cannot obtain more than $\ell$ new one-bits in the winner $x'$ of the mutation phase. Therefore, after we have sampled $\lambda$, the expected progress is
	\begin{align*}
		E[f(y) - f(x) \mid \lambda] \le E[\ell \mid \lambda] = \lambda.
	\end{align*}

	The expected progress in one iteration thus is 
	\begin{align*}
		E[f(y) - f(x)] = \sum_{i = 1}^u \Pr[\lambda = i] E[f(y) - f(x) \mid \lambda = i] \le E[\lambda].
	\end{align*}

	Let $x_0$ be the initial individual. Since it is chosen uniformly at random, its expected fitness is $E[f(x_0)] = \frac{n}{2}$. Hence, by the additive drift theorem~\cite{HeY01} the expectation of the number of iterations $T_I$ before the algorithm finds the optimum is at least 
	\begin{align*}
		E[T_I] \ge \frac{n - E[f(x_0)]}{E[\lambda]} = \frac{n}{2E[\lambda]}.
	\end{align*}

	Now we can use Wald's equation as we did in the proof of Theorem~\ref{thm:fitness-evaluations}. We obtain
	\begin{align*}
		E[T_F] = E[T_I]E[2\lambda] \ge \frac{n}{2 E[\lambda]} \cdot 2E[\lambda] = n. & \qedhere
	\end{align*}

\end{proof}

\section{Experiments}
\label{sec:experiments}
\newcommand{\thefontsize}{\tiny}

Our theoretical findings show that the fast \ollea with the natural choice $\beta \in (2,3)$ has a linear runtime on \onemax, which matches the performance of the self-adjusting \ollea.
Due to their asymptotic nature, our results cannot indicate which of the two linear-time algorithms is faster, how the fast \ollea compares with other algorithms on reasonable problem sizes,
and how its performance depends on $\beta \in (2,3)$. For the latter, the only estimate we have from theory,
eq.~\eqref{leading-constant-estimation}, provides a very large upper bound on the constant factor, which could suggest that $\beta = 2.5 + \varepsilon$ may be better than $\beta = 2.5 - \varepsilon$
for $0 < \varepsilon < 0.5$, but without a matching lower bound this is speculative. To answer these questions, but also to investigate the performance on a slightly less artificial problem, we performed a series of experiments.

As algorithms, we regarded randomized local search (RLS) and the \oea with a standard bit mutation
as well as the self-adjusting \ollea, which controls $\lambda$ (and thus $p = \lambda / n$ and $c = 1/\lambda$) via the one-fifth success rule~\cite{DoerrD18}.

We have also considered the version of the \ollea with the one-fifth rule with an upper limit of $2 \ln (n+1)$ on the value of $\lambda$,
introduced in~\cite{BuzdalovD17}, since it showed a much better performance on the \maxsat problem than without this upper limit.
For the same reason, we also consider the fast \ollea with the same upper limit of $2 \ln (n+1)$ on the value of $\lambda$,
which is imposed by setting the distribution parameter $u$ to $u = 2 \ln (n+1)$).
To investigate the effect of varying $u$ further, we also conduct a series of experiments with a fixed problem size $n$ and different values of $u$.

For the fast \ollea, we used the values of $\beta \in \{2.1, 2.3, 2.5, 2.7, 2.9\}$
unless noted otherwise.
In all the adaptive versions of the \ollea, the initial value of $\lambda$ is set to 1.

The source code used to perform these experiments is a part of a larger project dedicated to the \ollea available on GitHub\footnote{\url{https://github.com/mbuzdalov/generic-onell}} and as the supplementary material for this paper.

\subsection{Implementation Details and Their Discussion}

In all runs we used slightly modified versions of the algorithms to avoid counting obviously unnecessary fitness evaluations. 
The particular changes are as follows.
\begin{itemize}
    \item In the \oea, if standard bit mutation flips zero bits, then we resample the offspring until it is different from the parent. This is equivalent to not counting the fitness evaluation of the offspring identical to the parent.
    \item In all versions of the \ollea, we resample $\ell$ until $\ell \ne 0$. This is equivalent to not counting the fitness evaluations in iterations with $\ell = 0$ because here all offspring are identical to the parent. In the crossover phase, samples taking all bits from the parent $x$ are repeated (without evaluating the fitness of the copy of the parent) and samples taking all bits from the mutation winner $x'$ are not evaluated (that is, do not count towards the number of fitness evaluations).
Additionally, $x'$ also participates in the selection of the best among $x$ and the crossover results~$y^{(i)}$. When there is a tie, then the crossover winner has a higher priority than $x'$.
\end{itemize}

We consider these natural modifications instead of the original algorithms in this section, since we are sure that anyone implementing these algorithms for solving practical problems would do the same. For a practitioner it does not make sense to waste fitness evaluations on individuals which are identical to their parents, while in theoretical works these are often counted since constant factors are often ignored. We note that similar modifications of algorithms were called \emph{practice-aware} in~\cite{practice-aware}. We note that there are much more ways to tune the runtime of the \ollea in a practical application, see, e.g.,~\cite{GoldmanP14}. In contrast to the modifications described above, for these it is not clear to what extent they are useful in general or only for particular problems. For this reason, we did not consider them in this work.

Clearly our theoretical results from Section~\ref{sec:analysis} apply to these mildly modified algorithms. For the upper bounds it is enough to note that by resampling identical individuals and by having $x'$ participate in the selection, the probability to have a progress in one iteration only increases. Thus, repeating the arguments from Theorem~\ref{thm:iterations} we obtain the same upper bound on the expected number of iterations. Since our implementation does not affect the choice of $\lambda$, its expected value $E[\lambda]$ stays the same. The cost of one iteration is at most $2\lambda$ (but can be smaller). Thus, by Wald's equation we obtain the same upper bound on the expected number of fitness evaluations as in Theorem~\ref{thm:fitness-evaluations}. For the lower bound we use the same arguments as in Theorem~\ref{thm:lower-bound}, with the only change that since we cannot choose $\ell = 0$, we have
\begin{align*}
	E[\ell \mid \lambda] = \frac{\lambda}{1 - \left(1 - \frac 1\lambda\right)^\lambda} \le \frac{\lambda}{1 - \frac{1}{e}},
\end{align*}
which still gives us a lower bound of $\Omega(n)$ fitness evaluations.

\subsection{Experimental Setup}

The experiments were performed on the \onemax function and on random satisfiable instances of the \maxsat problem, that is, the problem of maximizing the number of satisfied clauses in a Boolean formula represented in conjunctive normal form. The second problem was chosen for two reasons. First, it is a more practical problem than \onemax, second, there are already theoretical and empirical results for the \ollea on this function (see~\cite{BuzdalovD17}). For this problem on $n$ variables, the number of clauses was chosen to be $4 n \ln n$. An all-ones bit string is assumed to be a planted optimal solution; this is without loss of generality, as all considered algorithms are unbiased. For each clause, three participating variables and their signs (i.e.,~whether it is negated or not) are sampled uniformly and independently until this clause is satisfied by the planted solution (that is, not all three variables are negated). Note that these are easy instances of the \maxsat problem, so the presented results on this problem should not be considered as if the proposed algorithms are competitive in solving this problem in general. However, these instances have a lower fitness-distance correlation, which makes them harder in particular for the \ollea.

To speed-up the experiments, we used the incremental fitness evaluation technique,
which is more commonly seen in gray-box optimization and in problem-aware solvers. We note that this led only to a faster implementation of the algorithm, not to a different algorithm behavior. In particular, the number of iterations or fitness evaluations performed are not affected. We modified the implementation as follows.

During mutation we do not copy the parent individual,
but instead directly generate the bit indices which are different in the parent and the offspring (the ``patch''). Following that, we evaluate the fitness of the offspring
based on the fitness of the parent and the patch. For RLS and the $(1+1)$~EA, if the new fitness is at least as good as the one of the parent, we apply the patch to the parent,
turning it into the offspring. For the \ollea, we select the best patch out of all the mutants' patches (based on their fitness values).
The subsequent applications of crossover translate to subsamplings of that patch, so that fitness evaluation is again based on the parent's fitness.

For \onemax, evaluation of the offspring's fitness based on the parent's fitness and the patch is rather straightforward: only the bits at the affected indices are checked.
This results in an expected $O(1)$ amount of work per each iteration of both RLS and the $(1+1)$~EA, and in the $\Theta(\lambda^2)$ amount of work for the \ollea,
which still helps much because $\lambda$ is typically much smaller than $n$.

For \maxsat, the incremental evaluation is more difficult as it involves some preprocessing on the side of the fitness function.
It amounts to constructing lists of clauses affected by the changed bits and to evaluating the satisfaction status of these clauses before
and after the change. For the logarithmic density of clauses (that is, the ratio of the number of clauses to the number of variables) employed in this paper, this amounts to $\Theta(\log n)$ expected work per iteration of RLS and the $(1+1)$~EA,
and to $\Theta(\lambda^2\log n)$ expected work for the \ollea, which is still faster than direct evaluation, but less efficient than what is possible for \onemax.

We also note that the particular structure of all the considered algorithms also allows to optimize the memory requirements:
the memory used by RLS and the $(1+1)$~EA is $\Theta(n)$ words resulting from storing a single bit vector,
whereas the \ollea uses $\Theta(n + \lambda)$ words in expectation, as only the best patches for each of the phases need to be stored.

In our experiments we chose the problem sizes $n$ to be powers of two, so that the asymptotic behavior of the algorithms is easier to investigate visually.
For \onemax, we limit the problem size to $2^{22}$, and for \maxsat, the upper limit is $2^{16}$. These sizes were derived from the affordable computational times.
We did not reach the size of $2^{20}$ on \maxsat as in~\cite{BuzdalovD17}, because the incremental fitness evaluations have a weaker impact with fast mutation.
Indeed, whenever $\lambda$ is sampled from a heavy-tailed distribution, the distribution of $\lambda^2$, and hence of the wall-clock running time, has an even heavier tail,
so occasional high values of $\lambda$ result in very expensive iterations.
For each algorithm, each problem setting, and each problem size, 100 independent runs were performed. For the \maxsat problem, a new random instance was created for each run.

Our runtime results are shown in Figures~\ref{exp:onemax}-\ref{exp:maxsat:varcap}. In Figures~\ref{exp:onemax}-\ref{exp:maxsat:capped} the $x$-axis indicates the problem size in a logarithmic scale,
and the $y$-axis indicates the ratio of the runtime to the problem size. In this visualization a linear runtime results in a horizontal plot
and any runtime in $\Theta(n \log n)$ gives a linearly increasing plot.

\subsection{Runtimes on \onemax}

\begin{figure}[!t]
\begin{tikzpicture}
\begin{axis}[width=1\linewidth, height=0.5\textheight, xmode=log, log base x=2, grid=major, ymin=0,
             xlabel={Problem size $n$}, ylabel={Evaluations / $n$},
			 legend pos=north west, cycle list name=myplotcycle,
			 xmin = 20, xmax = 6291456, ymax = 25,
             every axis plot/.append style={thick}]
			 \addplot plot [error bars/.cd, y dir=both, y explicit] coordinates {(32,5.034)+-(0,1.384)(64,5.349)+-(0,1.348)(128,6.065)+-(0,1.183)(256,6.14)+-(0,0.8272)(512,6.481)+-(0,0.6865)(1024,6.545)+-(0,0.4771)(2048,6.685)+-(0,0.4268)(4096,6.939)+-(0,0.4276)(8192,6.992)+-(0,0.3268)(16384,7.139)+-(0,0.3051)(32768,7.184)+-(0,0.2784)(65536,7.268)+-(0,0.2833)(131072,7.358)+-(0,0.2867)(262144,7.455)+-(0,0.2917)(524288,7.497)+-(0,0.2314)(1048576,7.547)+-(0,0.2273)(2097152,7.581)+-(0,0.2203)(4194304,7.696)+-(0,0.2407)};
			 \addlegendentry{\thefontsize $\lambda\in [1..2\ln (n + 1)]$};
			 \addplot plot [error bars/.cd, y dir=both, y explicit] coordinates {(32,4.897)+-(0,1.976)(64,5.67)+-(0,1.352)(128,5.886)+-(0,0.9753)(256,6.231)+-(0,0.7924)(512,6.343)+-(0,0.6007)(1024,6.339)+-(0,0.4546)(2048,6.507)+-(0,0.2927)(4096,6.566)+-(0,0.2452)(8192,6.61)+-(0,0.199)(16384,6.633)+-(0,0.136)(32768,6.638)+-(0,0.08875)(65536,6.661)+-(0,0.06778)(131072,6.678)+-(0,0.05667)(262144,6.683)+-(0,0.03487)(524288,6.688)+-(0,0.02944)(1048576,6.695)+-(0,0.02066)(2097152,6.698)+-(0,0.01393)(4194304,6.696)+-(0,0.01111)};
			 \addlegendentry{\thefontsize $\lambda\in [1..n]$};
			 \addplot plot [error bars/.cd, y dir=both, y explicit] coordinates {(32,6.564)+-(0,1.98)(64,7.556)+-(0,1.938)(128,8.47)+-(0,1.658)(256,9.692)+-(0,1.504)(512,11.61)+-(0,1.986)(1024,11.99)+-(0,1.434)(2048,12.9)+-(0,1.517)(4096,13.69)+-(0,1.317)(8192,15.12)+-(0,1.567)(16384,15.85)+-(0,1.453)(32768,16.35)+-(0,1.384)(65536,17.18)+-(0,1.276)(131072,17.67)+-(0,1.404)(262144,18.32)+-(0,1.285)(524288,18.84)+-(0,1.183)(1048576,19.16)+-(0,1.236)};
			 \addlegendentry{\thefontsize $\lambda\sim\text{pow}(2.1)$};
			 \addplot plot [error bars/.cd, y dir=both, y explicit] coordinates {(32,5.83)+-(0,2.023)(64,6.652)+-(0,1.799)(128,8.397)+-(0,1.732)(256,8.981)+-(0,1.586)(512,9.356)+-(0,1.341)(1024,10.18)+-(0,1.243)(2048,10.67)+-(0,1.104)(4096,11.15)+-(0,0.9422)(8192,11.66)+-(0,0.9335)(16384,11.73)+-(0,0.7974)(32768,12.08)+-(0,0.7509)(65536,12.2)+-(0,0.8012)(131072,12.44)+-(0,0.5857)(262144,12.63)+-(0,0.5652)(524288,12.64)+-(0,0.468)(1048576,12.74)+-(0,0.4199)(2097152,12.79)+-(0,0.4102)};
			 \addlegendentry{\thefontsize $\lambda\sim\text{pow}(2.3)$};
			 \addplot plot [error bars/.cd, y dir=both, y explicit] coordinates {(32,5.623)+-(0,1.995)(64,6.562)+-(0,1.933)(128,7.505)+-(0,1.601)(256,8.128)+-(0,1.3)(512,8.864)+-(0,1.529)(1024,9.397)+-(0,1.269)(2048,9.924)+-(0,0.9907)(4096,10.16)+-(0,0.9355)(8192,10.45)+-(0,0.6846)(16384,10.74)+-(0,0.6981)(32768,10.96)+-(0,0.5783)(65536,11.07)+-(0,0.6053)(131072,11.24)+-(0,0.4516)(262144,11.25)+-(0,0.2744)(524288,11.43)+-(0,0.2625)(1048576,11.5)+-(0,0.2572)(2097152,11.58)+-(0,0.2175)(4194304,11.64)+-(0,0.1711)};
			 \addlegendentry{\thefontsize $\lambda\sim\text{pow}(2.5)$};
			 \addplot plot [error bars/.cd, y dir=both, y explicit] coordinates {(32,5.62)+-(0,2.157)(64,6.659)+-(0,1.696)(128,7.216)+-(0,1.665)(256,8.072)+-(0,1.602)(512,8.917)+-(0,1.384)(1024,9.177)+-(0,1.177)(2048,9.574)+-(0,0.9223)(4096,10.07)+-(0,0.9517)(8192,10.39)+-(0,0.6924)(16384,10.79)+-(0,0.7591)(32768,11.15)+-(0,0.7254)(65536,11.44)+-(0,0.5174)(131072,11.53)+-(0,0.4989)(262144,11.76)+-(0,0.4677)(524288,11.92)+-(0,0.3562)(1048576,12.12)+-(0,0.3503)(2097152,12.25)+-(0,0.3292)(4194304,12.41)+-(0,0.319)};
			 \addlegendentry{\thefontsize $\lambda\sim\text{pow}(2.7)$};
			 \addplot plot [error bars/.cd, y dir=both, y explicit] coordinates {(32,5.352)+-(0,1.79)(64,6.262)+-(0,2.147)(128,7.035)+-(0,1.872)(256,7.88)+-(0,1.587)(512,8.64)+-(0,1.45)(1024,9.093)+-(0,1.223)(2048,9.788)+-(0,1.377)(4096,10.35)+-(0,1.123)(8192,11.05)+-(0,1.014)(16384,11.5)+-(0,1.189)(32768,11.84)+-(0,0.8389)(65536,12.3)+-(0,0.9423)(131072,12.47)+-(0,0.6984)(262144,13.01)+-(0,0.6305)(524288,13.41)+-(0,0.7717)(1048576,13.7)+-(0,0.6636)(2097152,14.01)+-(0,0.6249)(4194304,14.18)+-(0,0.6367)};
			 \addlegendentry{\thefontsize $\lambda\sim\text{pow}(2.9)$};
			 \addplot plot [error bars/.cd, y dir=both, y explicit] coordinates {(32,4.734)+-(0,1.413)(64,6.022)+-(0,2.072)(128,7.082)+-(0,2.2)(256,8.128)+-(0,2.017)(512,9.287)+-(0,2.104)(1024,10.64)+-(0,2.216)(2048,11.79)+-(0,1.836)(4096,12.94)+-(0,2.013)(8192,14.9)+-(0,2.771)(16384,15.9)+-(0,2.362)(32768,16.81)+-(0,2.072)(65536,17.72)+-(0,2.071)(131072,19.14)+-(0,2.189)(262144,20.5)+-(0,2.747)(524288,21.26)+-(0,2.337)(1048576,22.63)+-(0,2.252)(2097152,23.64)+-(0,2.432)(4194304,25.46)+-(0,2.464)};
			 \addlegendentry{\thefontsize (1+1) EA};
			 \addplot plot [error bars/.cd, y dir=both, y explicit] coordinates {(32,3.358)+-(0,1.154)(64,4.005)+-(0,1.155)(128,4.771)+-(0,1.219)(256,5.453)+-(0,1.191)(512,6.094)+-(0,1.196)(1024,6.762)+-(0,1.004)(2048,7.472)+-(0,1.123)(4096,8.072)+-(0,1.281)(8192,8.738)+-(0,1.151)(16384,9.676)+-(0,1.408)(32768,10.15)+-(0,1.169)(65536,10.89)+-(0,1.246)(131072,11.48)+-(0,1.186)(262144,12.46)+-(0,1.341)(524288,13.26)+-(0,1.385)(1048576,13.77)+-(0,1.008)(2097152,14.51)+-(0,1.43)(4194304,15.16)+-(0,1.438)};
			 \addlegendentry{\thefontsize RLS};
             \addplot plot [error bars/.cd, y dir=both, y explicit] coordinates {(32,5.41)+-(0,1.318)(64,5.958)+-(0,1.204)(128,6.753)+-(0,0.9653)(256,7.325)+-(0,1.092)(512,7.709)+-(0,0.8988)(1024,8.123)+-(0,0.9054)(2048,8.578)+-(0,1.068)(4096,9.005)+-(0,0.8088)(8192,9.384)+-(0,0.7802)(16384,9.811)+-(0,0.7353)(32768,10.18)+-(0,0.6723)(65536,10.41)+-(0,0.5078)(131072,10.8)+-(0,0.6717)(262144,11.08)+-(0,0.5191)(524288,11.48)+-(0,0.6667)(1048576,11.8)+-(0,0.6183)(2097152,11.97)+-(0,0.5923)(4194304,12.35)+-(0,0.6031)};
			 \addlegendentry{\thefontsize $\lambda=2\sqrt{\frac{\lnp n \lnp \lnp n}{\lnp\lnp\lnp n}}$};
\end{axis}
\end{tikzpicture}
\caption{Mean runtimes and their standard deviation of different algorithms on \onemax benchmark problem. By $\lambda \in [1..u]$ we denote the self-adjusting parameter choice via the one-fifth rule in the interval $[1..u]$. The indicated interval for each value $X$ is $[E[X] - \sigma(X), E[X] + \sigma(X)]$, where $\sigma(X)$ is the standard deviation of~$X$. We write $\lnp x := \ln (x+1)$. By $pow(x)$ we denote the power-law distribution with parameters $u = n$ and $\beta = x$.}
\label{exp:onemax}
\end{figure}
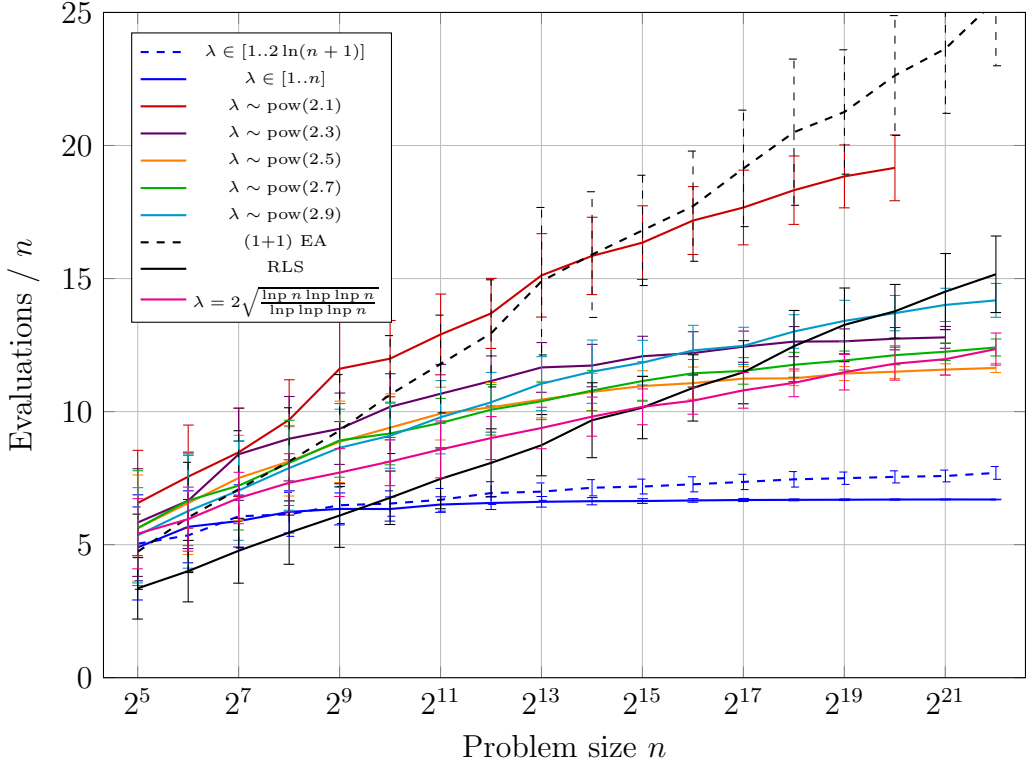

In Figure~\ref{exp:onemax} we show the results of the runs on the \onemax function. If we do not consider $\beta = 2.1$,
which turns out to be too small (and therefore gives a too large expected value of $\lambda$), then all versions of the fast \ollea start outperforming the \oea
already at population size $n = 2^{10}$ and then outperform RLS at $n = 2^{20}$ or earlier. Recalling the discussion after the proof of Theorem~\ref{thm:fitness-evaluations}
we note that our estimate of the leading constant in the runtime was overly pessimistic, otherwise we would have no chance to outperform RLS on these problem sizes.

The one-fifth rule shows a much better performance and yields a runtime of the \ollea which is very close to linear already from $n = 2^{10}$ on
for both linear and logarithmic upper bounds on $\lambda$. The plots for the heavy-tailed choice of $\lambda$ do not look horizontal,
but they show a strongly marked tendency that they will do so at larger population sizes. The runtimes for all $\beta$ except $\beta = 2.1$ are quite well concentrated,
as well as the runtimes of the \ollea with the one-fifth rule, in contrast to the runtimes of the \oea and RLS.
We have no results for $\beta = 2.1$ for population sizes $n \ge 2^{21}$ and for $\beta = 2.3$ for $n \ge 2^{22}$,
since they were too expensive (in terms of computational resources) and most likely not too insightful. 

Figure~\ref{exp:onemax} also features the runtime plot of an asymptotically optimal static choice for $\lambda$.
It has been proven in~\cite{DoerrD18} that the theoretically asymptotically optimal static choice is $\lambda = \Theta(\sqrt{\frac{\ln(n)\ln\ln(n)}{\ln\ln\ln(n)}})$.
By using $\lnp(n) := \ln(n+1)$ instead to avoid issues with logarithms of too small values,
and by fitting the outer constant factor using auxiliary experiments with fixed $\lambda \in [2..12]$,
we have found that $\lambda = 2\sqrt{\frac{\lnp(n)\lnp\lnp(n)}{\lnp\lnp\lnp(n)}}$
approximates the optimal choices quite well, so we have used the version of the \ollea with this choice in our plots.
We also see that with the choice of $\beta = 2.5$ the fast \ollea outperforms the statically optimal parameter choice at problem sizes $n \ge 2^{20}$.

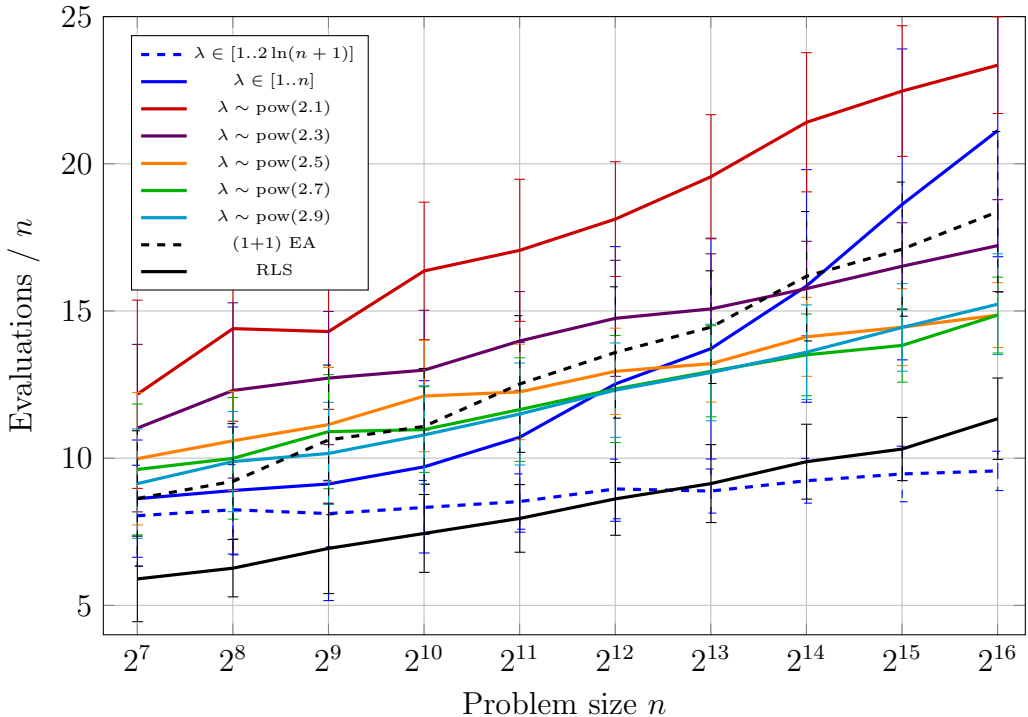
\begin{figure}[!t]
\begin{tikzpicture}
\begin{axis}[width=\linewidth, height=0.47\textheight, xmode=log, log base x=2, grid=major,
             xlabel={Problem size $n$}, ylabel={Evaluations / $n$},
             legend pos=north west, cycle list name=myplotcycle,
			 xmin = 100, xmax = 80000, ymin = 4, ymax = 25,
             every axis plot/.append style={very thick}]
			 \addplot plot [error bars/.cd, y dir=both, y explicit] coordinates {(128,8.044)+-(0,1.715)(256,8.246)+-(0,1.533)(512,8.117)+-(0,1.128)(1024,8.326)+-(0,0.9179)(2048,8.526)+-(0,0.9402)(4096,8.954)+-(0,1.013)(8192,8.882)+-(0,0.7489)(16384,9.233)+-(0,0.762)(32768,9.463)+-(0,0.9435)(65536,9.567)+-(0,0.6671)};
			 \addlegendentry{\thefontsize $\lambda\in [1..2\ln (n + 1)]$};
			 \addplot plot [error bars/.cd, y dir=both, y explicit] coordinates {(128,8.624)+-(0,1.992)(256,8.9)+-(0,2.159)(512,9.122)+-(0,3.96)(1024,9.705)+-(0,2.927)(2048,10.71)+-(0,3.225)(4096,12.52)+-(0,4.665)(8192,13.72)+-(0,3.752)(16384,15.85)+-(0,3.951)(32768,18.62)+-(0,5.281)(65536,21.13)+-(0,4.287)};
			 \addlegendentry{\thefontsize $\lambda\in [1..n]$};
			 \addplot plot [error bars/.cd, y dir=both, y explicit] coordinates {(128,12.17)+-(0,3.198)(256,14.4)+-(0,3.146)(512,14.3)+-(0,2.643)(1024,16.36)+-(0,2.335)(2048,17.06)+-(0,2.412)(4096,18.12)+-(0,1.949)(8192,19.56)+-(0,2.108)(16384,21.41)+-(0,2.364)(32768,22.47)+-(0,2.22)(65536,23.35)+-(0,1.644)};
			 \addlegendentry{\thefontsize $\lambda\sim\text{pow}(2.1)$};
			 \addplot plot [error bars/.cd, y dir=both, y explicit] coordinates {(128,11.02)+-(0,2.845)(256,12.3)+-(0,2.977)(512,12.72)+-(0,2.262)(1024,12.99)+-(0,2.034)(2048,13.98)+-(0,1.678)(4096,14.75)+-(0,1.971)(8192,15.07)+-(0,1.871)(16384,15.76)+-(0,1.606)(32768,16.52)+-(0,1.477)(65536,17.22)+-(0,1.561)};
			 \addlegendentry{\thefontsize $\lambda\sim\text{pow}(2.3)$};
			 \addplot plot [error bars/.cd, y dir=both, y explicit] coordinates {(128,9.982)+-(0,2.25)(256,10.59)+-(0,1.681)(512,11.14)+-(0,1.939)(1024,12.11)+-(0,1.891)(2048,12.25)+-(0,1.623)(4096,12.95)+-(0,1.465)(8192,13.21)+-(0,1.307)(16384,14.12)+-(0,1.341)(32768,14.45)+-(0,1.305)(65536,14.86)+-(0,1.102)};
			 \addlegendentry{\thefontsize $\lambda\sim\text{pow}(2.5)$};
			 \addplot plot [error bars/.cd, y dir=both, y explicit] coordinates {(128,9.616)+-(0,2.222)(256,9.993)+-(0,2.066)(512,10.9)+-(0,1.938)(1024,10.97)+-(0,1.454)(2048,11.65)+-(0,1.76)(4096,12.35)+-(0,1.819)(8192,12.95)+-(0,1.546)(16384,13.51)+-(0,1.388)(32768,13.83)+-(0,1.25)(65536,14.86)+-(0,1.288)};
			 \addlegendentry{\thefontsize $\lambda\sim\text{pow}(2.7)$};
			 \addplot plot [error bars/.cd, y dir=both, y explicit] coordinates {(128,9.135)+-(0,1.858)(256,9.884)+-(0,1.703)(512,10.16)+-(0,1.736)(1024,10.79)+-(0,1.673)(2048,11.5)+-(0,1.732)(4096,12.31)+-(0,1.603)(8192,12.91)+-(0,1.64)(16384,13.6)+-(0,1.607)(32768,14.44)+-(0,1.487)(65536,15.23)+-(0,1.712)};
			 \addlegendentry{\thefontsize $\lambda\sim\text{pow}(2.9)$};
			 \addplot plot [error bars/.cd, y dir=both, y explicit] coordinates {(128,8.631)+-(0,2.299)(256,9.211)+-(0,1.961)(512,10.62)+-(0,2.542)(1024,11.08)+-(0,1.967)(2048,12.52)+-(0,2.322)(4096,13.59)+-(0,2.231)(8192,14.45)+-(0,1.914)(16384,16.18)+-(0,2.198)(32768,17.1)+-(0,2.278)(65536,18.37)+-(0,2.727)};
			 \addlegendentry{\thefontsize (1+1) EA};
			 \addplot plot [error bars/.cd, y dir=both, y explicit] coordinates {(128,5.897)+-(0,1.453)(256,6.262)+-(0,0.9742)(512,6.934)+-(0,1.535)(1024,7.443)+-(0,1.322)(2048,7.953)+-(0,1.149)(4096,8.616)+-(0,1.236)(8192,9.136)+-(0,1.324)(16384,9.88)+-(0,1.272)(32768,10.31)+-(0,1.073)(65536,11.34)+-(0,1.382)};
			 \addlegendentry{\thefontsize RLS};			 
\end{axis}
\end{tikzpicture}
\caption{Mean runtimes and their standard deviation of different algorithms on \maxsat instances with $4n\ln(n)$ clauses. By $\lambda \in [1..u]$ we denote the self-adjusting parameter choice via the one-fifth rule in the interval $[1..u]$. The indicated interval for each value $X$ is $[E[X] - \sigma(X), E[X] + \sigma(X)]$, where $\sigma(X)$ is the standard deviation of $X$. By $pow(x)$ we denote the power-law distribution with parameters $u = n$ and $\beta = x$.}
\label{exp:maxsat}
\end{figure}

\subsection{{Runtimes on \maxsat}{}}

Figure~\ref{exp:maxsat} shows the results of the experiments on the \maxsat problem. As previously shown in~\cite{BuzdalovD17}, large values of $\lambda$ can be harmful.
For this reason, the \ollea with the unbounded one-fifth rule is outperformed already by the simple \oea.
The authors of~\cite{BuzdalovD17} proposed to limit the value which $\lambda$ can take by $2\ln(n + 1)$, which greatly improved the performance up to the point that RLS was outperformed on this problem.

As we see in Figure~\ref{exp:maxsat}, the fast \ollea is quite efficient even without an upper limit on $\lambda$. Except for the case $\beta = 2.1$, we managed to outperform the \oea and the self-adjusting \ollea without an upper limit on $\lambda$. Nevertheless, RLS and the self-adjusting \ollea with a logarithmic cap on $\lambda$ remained faster. 

The runtimes of all algorithms appear super-linear in the plots.

\begin{figure}[!t]
\newcommand{\htcaption}[1]{$\beta=#1, u=2\ln(n+1)$}
\begin{tikzpicture}
\begin{axis}[width=\linewidth, height=0.5\textheight, xmode=log, log base x=2, grid=major, ymin=4,
             xlabel={Problem size $n$}, ylabel={Evaluations / $n$},
             legend pos=north west, cycle list name=myplotcycle2,
			 xmin = 100, xmax = 780000,
             every axis plot/.append style={very thick}]
\addplot plot [error bars/.cd, y dir=both, y explicit] coordinates {(128,7.68)+-(0,1.407)(256,7.838)+-(0,1.307)(512,8.039)+-(0,1.027)(1024,8.097)+-(0,0.9763)(2048,8.228)+-(0,0.7573)(4096,8.706)+-(0,0.9018)(8192,8.828)+-(0,0.7874)(16384,8.919)+-(0,0.735)(32768,9.203)+-(0,0.6443)(65536,9.393)+-(0,0.6517)(131072,9.563)+-(0,0.6915)(262144,9.758)+-(0,0.7171)(524288,9.973)+-(0,0.6132)};
\addlegendentry{\thefontsize $\lambda \in [1..2\ln(n+1)]$};
\addplot plot [error bars/.cd, y dir=both, y explicit] coordinates {(128,8.787)+-(0,1.817)(256,9.762)+-(0,1.641)(512,9.868)+-(0,1.635)(1024,10.29)+-(0,1.534)(2048,10.76)+-(0,1.394)(4096,11.29)+-(0,1.426)(8192,11.69)+-(0,1.23)(16384,12.29)+-(0,1.078)(32768,12.67)+-(0,0.9932)(65536,13.25)+-(0,1.219)(131072,13.69)+-(0,1.191)(262144,13.95)+-(0,1.106)(524288,14.31)+-(0,1.236)};
\addlegendentry{\thefontsize \htcaption{2.1}};
\addplot plot [error bars/.cd, y dir=both, y explicit] coordinates {(128,9.155)+-(0,2.018)(256,9.256)+-(0,1.904)(512,9.837)+-(0,1.834)(1024,10.16)+-(0,1.508)(2048,10.83)+-(0,1.467)(4096,11.31)+-(0,1.185)(8192,11.99)+-(0,1.36)(16384,12.36)+-(0,1.225)(32768,13)+-(0,1.353)(65536,13.38)+-(0,1.267)(131072,14.2)+-(0,1.331)(262144,14.28)+-(0,1.023)(524288,14.83)+-(0,1.052)};
\addlegendentry{\thefontsize \htcaption{2.3}};
\addplot plot [error bars/.cd, y dir=both, y explicit] coordinates {(128,8.947)+-(0,2.44)(256,9.096)+-(0,1.688)(512,9.726)+-(0,1.806)(1024,10.51)+-(0,1.84)(2048,11.09)+-(0,1.695)(4096,11.7)+-(0,1.599)(8192,12.64)+-(0,1.582)(16384,12.79)+-(0,1.255)(32768,13.77)+-(0,1.579)(65536,14.06)+-(0,1.307)(131072,14.36)+-(0,1.301)(262144,14.96)+-(0,1.311)(524288,15.72)+-(0,1.399)};
\addlegendentry{\thefontsize \htcaption{2.5}};
\addplot plot [error bars/.cd, y dir=both, y explicit] coordinates {(128,8.614)+-(0,2.104)(256,9.242)+-(0,1.605)(512,10.11)+-(0,1.765)(1024,10.41)+-(0,1.561)(2048,11.4)+-(0,1.761)(4096,11.95)+-(0,1.723)(8192,12.63)+-(0,1.725)(16384,13.23)+-(0,1.36)(32768,13.91)+-(0,1.454)(65536,14.38)+-(0,1.416)(131072,15.26)+-(0,1.723)(262144,15.91)+-(0,1.463)(524288,16.57)+-(0,1.47)};
\addlegendentry{\thefontsize \htcaption{2.7}};
\addplot plot [error bars/.cd, y dir=both, y explicit] coordinates {(128,8.842)+-(0,2.264)(256,9.363)+-(0,2.011)(512,10.31)+-(0,1.892)(1024,10.73)+-(0,1.711)(2048,11.37)+-(0,1.79)(4096,12.42)+-(0,1.709)(8192,12.85)+-(0,1.57)(16384,13.55)+-(0,1.558)(32768,14.59)+-(0,1.759)(65536,15.37)+-(0,1.675)(131072,15.98)+-(0,1.598)(262144,16.72)+-(0,1.431)(524288,17.23)+-(0,1.635)};
\addlegendentry{\thefontsize \htcaption{2.9}};
\addplot plot [error bars/.cd, y dir=both, y explicit] coordinates {(128,8.397)+-(0,2.415)(256,9.283)+-(0,1.975)(512,10.2)+-(0,2.164)(1024,11.41)+-(0,2.099)(2048,12.62)+-(0,2.588)(4096,13.32)+-(0,2.084)(8192,14.75)+-(0,2.031)(16384,15.52)+-(0,1.778)(32768,16.92)+-(0,1.963)(65536,18.33)+-(0,2.167)(131072,19.27)+-(0,2.415)(262144,20.57)+-(0,2.245)(524288,21.34)+-(0,2.154)};
\addlegendentry{\thefontsize (1+1) EA};
\addplot plot [error bars/.cd, y dir=both, y explicit] coordinates {(128,5.716)+-(0,1.311)(256,6.336)+-(0,1.288)(512,6.844)+-(0,1.321)(1024,7.466)+-(0,1.31)(2048,7.86)+-(0,1.345)(4096,8.517)+-(0,1.289)(8192,9.118)+-(0,1.225)(16384,9.845)+-(0,1.377)(32768,10.58)+-(0,1.253)(65536,11.14)+-(0,1.132)(131072,12.11)+-(0,1.784)(262144,12.42)+-(0,1.099)(524288,13.11)+-(0,1.119)};
\addlegendentry{\thefontsize RLS};
\end{axis}
\end{tikzpicture}
\caption{Mean runtimes and their standard deviation of different algorithms on \maxsat instances with $4n\ln(n)$ clauses with logarithmically capped population sizes.
         By $\lambda \in [1..u]$ we denote the self-adjusting parameter choice via the one-fifth rule in the interval $[1..u]$.
         The indicated interval for each value $X$ is $[E[X] - \sigma(X), E[X] + \sigma(X)]$, where $\sigma(X)$ is the standard deviation of $X$.}
\label{exp:maxsat:capped}
\end{figure}
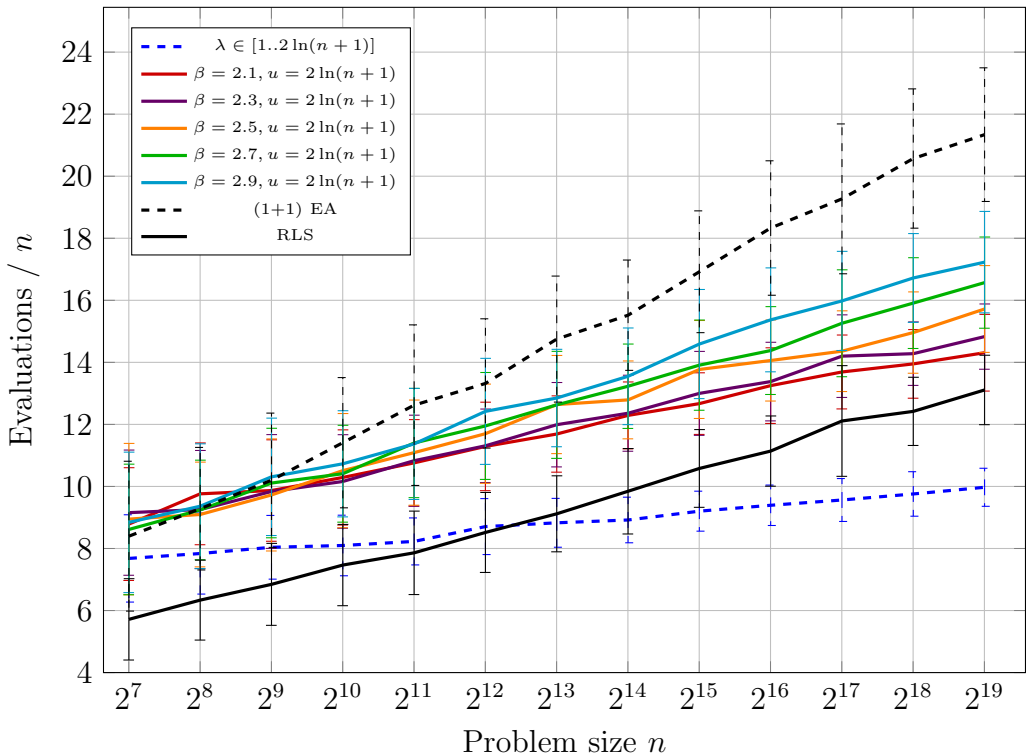

\subsection{Effects of Capping for \maxsat}

Since apparently large values of $\lambda$ are not helpful when optimizing \maxsat instances (due to the weaker fitness-distance correlation), we conducted some experiments with the fast \ollea choosing $\lambda$ from a power-law distribution on a smaller range $[1..u]$ of values. Based on the previous experience, we started with an upper limit of $u=2 \ln (n+1)$. These results are presented in Figure~\ref{exp:maxsat:capped}.

Using this upper limit reduced the computational burden associated with heavy-tailed distributions and allowed us to regard problem sizes up to $2^{19}$. The upper limit also led a better performance in terms of fitness evaluations. When comparing Figure~\ref{exp:maxsat} and Figure~\ref{exp:maxsat:capped} around the problem size $n=2^{16}$, we see that for $\beta \in \{2.1, 2.3\}$ a significant speed-up was obtained, whereas for $2.5 \le \beta \le 2.9$ the differences of the corresponding mean running times are negligible. This is not surprising given that for smaller values of~$\beta$, the inefficient high values of $\lambda$ are sampled more often. Interestingly, in combination with the upper limit small values of $\beta$ gave the best performance. This suggests that it is important to use moderately large values of $\lambda$ often and that only too large values lead to efficiency losses.

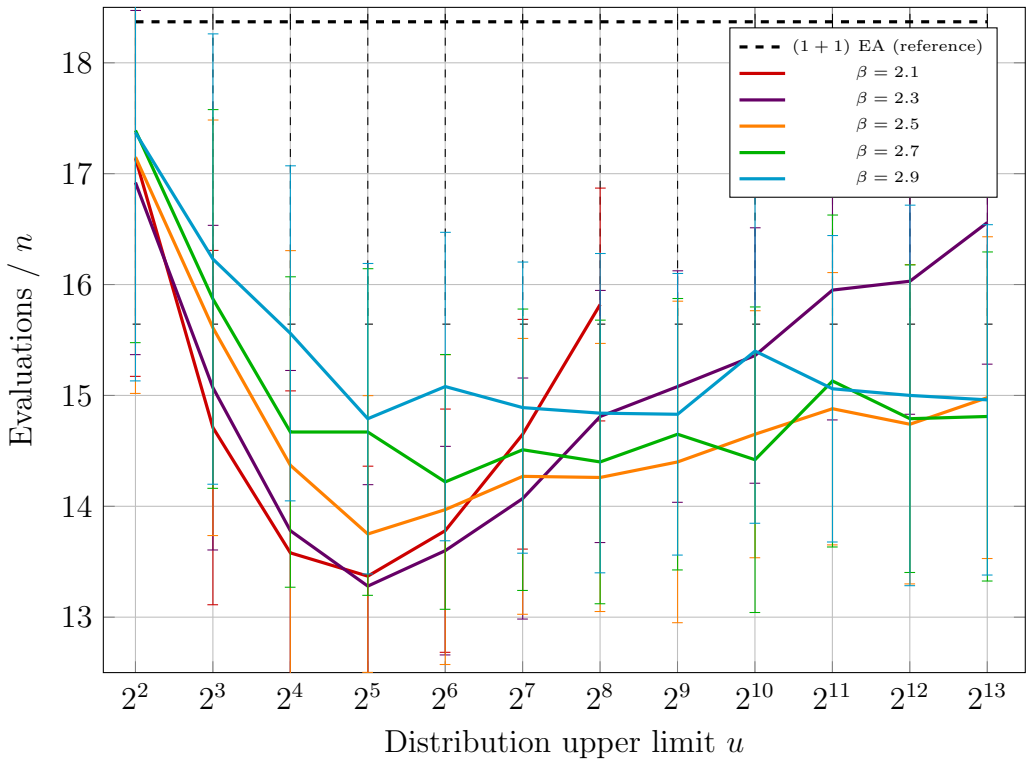
\begin{figure}[!t]
\begin{tikzpicture}
\begin{axis}[width=\linewidth, height=0.5\textheight, xmode=log, log base x=2, grid=major,
             xlabel={Distribution upper limit $u$}, ylabel={Evaluations / $n$},
             xmin=3, xmax=11500, ymin=12.5, ymax=18.5,
             legend pos=north east, cycle list name=myplotcycle3,
             every axis plot/.append style={very thick}]
\addplot plot [error bars/.cd, y dir=both, y explicit] coordinates {(4,18.37)+-(0,2.727)(8,18.37)+-(0,2.727)(16,18.37)+-(0,2.727)(32,18.37)+-(0,2.727)(64,18.37)+-(0,2.727)(128,18.37)+-(0,2.727)(256,18.37)+-(0,2.727)(512,18.37)+-(0,2.727)(1024,18.37)+-(0,2.727)(2048,18.37)+-(0,2.727)(4096,18.37)+-(0,2.727)(8192,18.37)+-(0,2.727)};
\addlegendentry{\thefontsize $(1+1)$ EA (reference)};
\addplot plot [error bars/.cd, y dir=both, y explicit] coordinates {(4,17.14)+-(0,1.968)(8,14.71)+-(0,1.598)(16,13.58)+-(0,1.461)(32,13.37)+-(0,0.9918)(64,13.78)+-(0,1.097)(128,14.65)+-(0,1.036)(256,15.82)+-(0,1.05)};
\addlegendentry{\thefontsize $\beta=2.1$};
\addplot plot [error bars/.cd, y dir=both, y explicit] coordinates {(4,16.92)+-(0,1.552)(8,15.07)+-(0,1.464)(16,13.78)+-(0,1.445)(32,13.28)+-(0,0.9151)(64,13.6)+-(0,0.9406)(128,14.07)+-(0,1.087)(256,14.81)+-(0,1.137)(512,15.08)+-(0,1.044)(1024,15.36)+-(0,1.152)(2048,15.95)+-(0,1.171)(4096,16.03)+-(0,1.2)(8192,16.56)+-(0,1.278)};
\addlegendentry{\thefontsize $\beta=2.3$};
\addplot plot [error bars/.cd, y dir=both, y explicit] coordinates {(4,17.15)+-(0,2.132)(8,15.61)+-(0,1.874)(16,14.37)+-(0,1.936)(32,13.75)+-(0,1.248)(64,13.97)+-(0,1.397)(128,14.27)+-(0,1.244)(256,14.26)+-(0,1.209)(512,14.4)+-(0,1.45)(1024,14.65)+-(0,1.114)(2048,14.88)+-(0,1.228)(4096,14.74)+-(0,1.44)(8192,14.98)+-(0,1.451)};
\addlegendentry{\thefontsize $\beta=2.5$};
\addplot plot [error bars/.cd, y dir=both, y explicit] coordinates {(4,17.39)+-(0,1.914)(8,15.87)+-(0,1.708)(16,14.67)+-(0,1.4)(32,14.67)+-(0,1.473)(64,14.22)+-(0,1.149)(128,14.51)+-(0,1.269)(256,14.4)+-(0,1.279)(512,14.65)+-(0,1.224)(1024,14.42)+-(0,1.378)(2048,15.13)+-(0,1.497)(4096,14.79)+-(0,1.387)(8192,14.81)+-(0,1.484)};
\addlegendentry{\thefontsize $\beta=2.7$};
\addplot plot [error bars/.cd, y dir=both, y explicit] coordinates {(4,17.37)+-(0,2.239)(8,16.23)+-(0,2.031)(16,15.56)+-(0,1.511)(32,14.79)+-(0,1.4)(64,15.08)+-(0,1.391)(128,14.89)+-(0,1.313)(256,14.84)+-(0,1.44)(512,14.83)+-(0,1.27)(1024,15.4)+-(0,1.553)(2048,15.06)+-(0,1.382)(4096,15)+-(0,1.716)(8192,14.96)+-(0,1.58)};
\addlegendentry{\thefontsize $\beta=2.9$};
\end{axis}
\end{tikzpicture}
\caption{Mean runtimes and their standard deviation of different algorithms on \maxsat instances with $4n\ln(n)$ clauses for different capping values. Problem size is $n=2^{16}$.
         The indicated interval for each value $X$ is $[E[X] - \sigma(X), E[X] + \sigma(X)]$, where $\sigma(X)$ is the standard deviation of $X$.}
\label{exp:maxsat:varcap}
\end{figure}

To investigate the effect of the particular choice of the upper limit $u$ on the running time for various values of $\beta$,
we performed additional experiments where the problem size was fixed to $n=2^{16}$, but the upper limits were varying.
Figure~\ref{exp:maxsat:varcap} presents these results, where $u$ was taken from the set $u \in \{ 2^2, 2^3, \ldots, 2^{13}\}$.
Note that high values of $u$ again prevented us from choosing a higher problem size. We also plot for reference the performance
of the $(1+1)$~EA on the same problem size.

The plots in Figure~\ref{exp:maxsat:varcap} indicate that for $2.1 \le \beta \le 2.5$ the dependency on the upper limit has a clear optimal value:
Too small values of $u$ prevent the \ollea from choosing the more efficient mid-size values of $\lambda$, too high values of $u$ lead to sampling too large values of $\lambda$ too often, which have little chance of making progress and at the same time are very costly.
It can be seen, however, that already for $\beta = 2.5$ the subsequent increase of the running time is not too pronounced.
Higher values of $\beta$ tend to a monotonic behavior, up to the deviations from the mean running time.
This basically indicates that the sensitivity to the upper limit of the distribution is not large even in practice.

\subsection{Summary of Experimental Results}

Summing up, from the results of the experiments we conclude the following three points.

\begin{itemize}
	\item The fast \ollea performs generally well, often beating the classic mutation based algorithms. On \onemax, the self-adjusting \ollea both without and with an upper limit of $u = 2 \ln(n+1)$ are superior, on \maxsat only the version with upper limit and RLS are superior.
	\item The fast \ollea can easily be used as a parameterless algorithm and this is what we suggest. We note that the \ollea with the asymptotically optimal static parameter setting could not beat the fast \ollea on \onemax. The self-adjusting \ollea without an upper limit was superior on \onemax, but significantly inferior on \maxsat. The version with upper limit $u = 2 \ln(n+1)$ was superior on both \onemax and \maxsat. We still do not want to advertise this approach as clearly such limits are problem-specific and non-trivial to find. The logarithmic limit for \maxsat is based on a substantial mathematical analysis~\cite{BuzdalovD17} of these particular \maxsat instances. For other problems, such a limit may be detrimental, e.g., it may be hard to leave a local optimum with a large basin of attraction.
	\item The choice of $\beta$ does not play a big role as long as it is not too close to the borders of the interval $(2, 3)$. Taking $\beta$ between $2.5$ and $2.7$ might be a good general recommendation.
\end{itemize}

\section{Conclusion}

In this first runtime analysis of a crossover-based algorithm using the fast mutation operator, we observed that the fast mutation operator not only can relieve the algorithm designer from the task of choosing a suitable mutation rate, but it can also lead to runtimes asymptotically better than any static choice of the mutation rate. 

Different from previous works, where any power-law exponent greater than one could be used, our work requires that $\beta$ is between $2$ and $3$. 
We note, however, that the power-law distributions are often used with exponents in the open interval $(2,3)$ and this for good reason. In this regime, we have a heavy tail (as opposed for $\beta > 3$), but we still have a constant expectation (as opposed to $\beta < 2$). Since the complexity of a single iteration is $\Theta(\lambda)$, having a constant expectation $E[\lambda]$ is very natural. 

On the technical side, our work shows that algorithms with a heavy-tailed number of offspring can be much easier to analyze than those with a self-adjusting number of offspring (such as the self-adjusting \ollea~\cite{DoerrD18}), since Wald's equation allows to estimate the expected runtime as the product of the expected number of iterations and the expected number of offspring generated in one iteration. 

The natural question arising from this work is for which other algorithms and problems such a speed-up can be obtained. Natural candidates are other crossover-based algorithms or algorithms in which dynamic parameter choices could obtain a speed-up over static choices. We note that after this research was conducted, it was found that the \ollea with two of its parameters chosen independently from heavy-tailed distributions has a good performance on jump functions~\cite{AntipovD20ppsn}. The performance is slightly inferior to the one with optimal static parameters~\cite{AntipovDK20}, however these were non-trivial to find as they deviated significantly from the previous recommendations.

\section*{Acknowledgements}
This work was supported by a public grant as part of the Investissement d'avenir project, reference ANR-11-LABX-0056-LMH, LabEx LMH and by RFBR and CNRS, project number 20-51-15009.


\begin{thebibliography}{DLMN17}

\bibitem[ABD20]{AntipovBD20gecco}
Denis Antipov, Maxim Buzdalov, and Benjamin Doerr.
\newblock Fast mutation in crossover-based algorithms.
\newblock In {\em Genetic and Evolutionary Computation Conference, GECCO 2020},
  pages 1268--1276. {ACM}, 2020.

\bibitem[AD11]{AugerD11}
Anne Auger and Benjamin Doerr, editors.
\newblock {\em Theory of Randomized Search Heuristics}.
\newblock World Scientific Publishing, 2011.

\bibitem[AD20]{AntipovD20ppsn}
Denis Antipov and Benjamin Doerr.
\newblock Runtime analysis of a heavy-tailed $(1+(\lambda, \lambda))$ genetic
  algorithm on jump functions.
\newblock In {\em Parallel Problem Solving From Nature, PPSN 2020, Part II},
  pages 545--559. Springer, 2020.

\bibitem[ADK19]{AntipovDK19foga}
Denis Antipov, Benjamin Doerr, and Vitalii Karavaev.
\newblock A tight runtime analysis for the ${(1 + (\lambda,\lambda))}$ {GA} on
  {Leading\-Ones}.
\newblock In {\em Foundations of Genetic Algorithms, FOGA 2019}, pages
  169--182. ACM, 2019.

\bibitem[ADK20]{AntipovDK20}
Denis Antipov, Benjamin Doerr, and Vitalii Karavaev.
\newblock The $(1 + (\lambda,\lambda))$ {GA} is even faster on multimodal
  problems.
\newblock In {\em Genetic and Evolutionary Computation Conference, GECCO 2020},
  pages 1259--1267. {ACM}, 2020.

\bibitem[B{\"{a}}c93]{Back93}
Thomas B{\"{a}}ck.
\newblock Optimal mutation rates in genetic search.
\newblock In {\em International Conference on Genetic Algorithms, ICGA 1993},
  pages 2--8. Morgan Kaufmann, 1993.

\bibitem[BD17]{BuzdalovD17}
Maxim Buzdalov and Benjamin Doerr.
\newblock Runtime analysis of the ${(1+(\lambda,\lambda))}$ genetic algorithm
  on random satisfiable 3-{CNF} formulas.
\newblock In {\em Genetic and Evolutionary Computation Conference, GECCO 2017},
  pages 1343--1350. {ACM}, 2017.

\bibitem[DD18]{DoerrD18}
Benjamin Doerr and Carola Doerr.
\newblock Optimal static and self-adjusting parameter choices for the
  ${(1+(\lambda,\lambda))}$ genetic algorithm.
\newblock {\em Algorithmica}, 80:1658--1709, 2018.

\bibitem[DDE15]{DoerrDE15}
Benjamin Doerr, Carola Doerr, and Franziska Ebel.
\newblock From black-box complexity to designing new genetic algorithms.
\newblock {\em Theoretical Computer Science}, 567:87--104, 2015.

\bibitem[DJS{\etalchar{+}}13]{DoerrJSWZ13}
Benjamin Doerr, Thomas Jansen, Dirk Sudholt, Carola Winzen, and Christine
  Zarges.
\newblock Mutation rate matters even when optimizing monotone functions.
\newblock {\em Evolutionary Computation}, 21:1--21, 2013.

\bibitem[DK15]{DoerrK15}
Benjamin Doerr and Marvin K{\"{u}}nnemann.
\newblock Optimizing linear functions with the $(1+\lambda)$ evolutionary
  algorithm---different asymptotic runtimes for different instances.
\newblock {\em Theoretical Computer Science}, 561:3--23, 2015.

\bibitem[DLMN17]{DoerrLMN17}
Benjamin Doerr, Huu~Phuoc Le, R\'egis Makhmara, and Ta~Duy Nguyen.
\newblock Fast genetic algorithms.
\newblock In {\em Genetic and Evolutionary Computation Conference, GECCO 2017},
  pages 777--784. {ACM}, 2017.

\bibitem[DN20]{DoerrN20}
Benjamin Doerr and Frank Neumann, editors.
\newblock {\em Theory of Evolutionary Computation---Recent Developments in
  Discrete Optimization}.
\newblock Springer, 2020.
\newblock Also available at
  \url{https://cs.adelaide.edu.au/~frank/papers/TheoryBook2019-selfarchived.pdf}.

\bibitem[Doe20a]{Doerr20gecco}
Benjamin Doerr.
\newblock Does comma selection help to cope with local optima?
\newblock In {\em Genetic and Evolutionary Computation Conference, GECCO 2020},
  pages 1304--1313. {ACM}, 2020.

\bibitem[Doe20b]{Doerr20bookchapter}
Benjamin Doerr.
\newblock Probabilistic tools for the analysis of randomized optimization
  heuristics.
\newblock In Benjamin Doerr and Frank Neumann, editors, {\em Theory of
  Evolutionary Computation: Recent Developments in Discrete Optimization},
  pages 1--87. Springer, 2020.
\newblock Also available at \url{https://arxiv.org/abs/1801.06733}.

\bibitem[GKS99]{GarnierKS99}
Josselin Garnier, Leila Kallel, and Marc Schoenauer.
\newblock Rigorous hitting times for binary mutations.
\newblock {\em Evolutionary Computation}, 7:173--203, 1999.

\bibitem[GP14]{GoldmanP14}
Brian~W. Goldman and William~F. Punch.
\newblock Parameter-less population pyramid.
\newblock In {\em Genetic and Evolutionary Computation Conference, {GECCO}
  2014}, pages 785--792. {ACM}, 2014.

\bibitem[GW17]{GiessenW17}
Christian Gie{\ss}en and Carsten Witt.
\newblock The interplay of population size and mutation probability in the ${(1
  + \lambda)}$ {EA} on {OneMax}.
\newblock {\em Algorithmica}, 78:587--609, 2017.

\bibitem[HY01]{HeY01}
Jun He and Xin Yao.
\newblock Drift analysis and average time complexity of evolutionary
  algorithms.
\newblock {\em Artificial Intelligence}, 127:51--81, 2001.

\bibitem[Jan13]{Jansen13}
Thomas Jansen.
\newblock {\em Analyzing Evolutionary Algorithms -- The Computer Science
  Perspective}.
\newblock Springer, 2013.

\bibitem[JJW05]{JansenJW05}
Thomas Jansen, Kenneth A.~De Jong, and Ingo Wegener.
\newblock On the choice of the offspring population size in evolutionary
  algorithms.
\newblock {\em Evolutionary Computation}, 13:413--440, 2005.

\bibitem[Leh10]{Lehre10}
Per~Kristian Lehre.
\newblock Negative drift in populations.
\newblock In {\em Parallel Problem Solving from Nature, PPSN 2010}, pages
  244--253. Springer, 2010.

\bibitem[Leh11]{Lehre11}
Per~Kristian Lehre.
\newblock Fitness-levels for non-elitist populations.
\newblock In {\em Genetic and Evolutionary Computation Conference, {GECCO}
  2011}, pages 2075--2082. {ACM}, 2011.

\bibitem[Len18]{Lengler18}
Johannes Lengler.
\newblock A general dichotomy of evolutionary algorithms on monotone functions.
\newblock In {\em Parallel Problem Solving from Nature, PPSN 2018, Part {II}},
  pages 3--15. Springer, 2018.

\bibitem[M{\"{u}}h92]{Muhlenbein92}
Heinz M{\"{u}}hlenbein.
\newblock How genetic algorithms really work: mutation and hillclimbing.
\newblock In {\em Parallel Problem Solving from Nature, PPSN 1992}, pages
  15--26. Elsevier, 1992.

\bibitem[NW10]{NeumannW10}
Frank Neumann and Carsten Witt.
\newblock {\em Bioinspired Computation in Combinatorial Optimization --
  Algorithms and Their Computational Complexity}.
\newblock Springer, 2010.

\bibitem[PD18]{practice-aware}
Eduardo~Carvalho Pinto and Carola Doerr.
\newblock Towards a more practice-aware runtime analysis of evolutionary
  algorithms.
\newblock {\em CoRR}, abs/1812.00493, 2018.

\bibitem[Pr{\"{u}}04]{Prugel04}
Adam Pr{\"{u}}gel{-}Bennett.
\newblock When a genetic algorithm outperforms hill-climbing.
\newblock {\em Theoretical Computer Science}, 320:135--153, 2004.

\bibitem[RS14]{RoweS14}
Jonathan~E. Rowe and Dirk Sudholt.
\newblock The choice of the offspring population size in the (1, {\(\lambda\)})
  evolutionary algorithm.
\newblock {\em Theoretical Computer Science}, 545:20--38, 2014.

\bibitem[SH87]{SzuH87}
Harold~H. Szu and Ralph~L. Hartley.
\newblock Fast simulated annealing.
\newblock {\em Physics Letters A}, 122:157--162, 1987.

\bibitem[TG06]{TeytaudG06}
Olivier Teytaud and Sylvain Gelly.
\newblock General lower bounds for evolutionary algorithms.
\newblock In {\em Parallel Problem Solving from Nature, PPSN 2006}, pages
  21--31. Springer, 2006.

\bibitem[Wal45]{Wald45}
Abraham Wald.
\newblock Some generalizations of the theory of cumulative sums of random
  variables.
\newblock {\em The Annals of Mathematical Statistics}, 16:287--293, 1945.

\bibitem[Wit06]{Witt06}
Carsten Witt.
\newblock Runtime analysis of the ($\mu$ + 1) {EA} on simple pseudo-{B}oolean
  functions.
\newblock {\em Evolutionary Computation}, 14:65--86, 2006.

\bibitem[Wit13]{Witt13}
Carsten Witt.
\newblock Tight bounds on the optimization time of a randomized search
  heuristic on linear functions.
\newblock {\em Combinatorics, Probability {\&} Computing}, 22:294--318, 2013.

\bibitem[YL97]{YaoL97}
Xin Yao and Yong Liu.
\newblock Fast evolution strategies.
\newblock In {\em Evolutionary Programming}, volume 1213 of {\em Lecture Notes
  in Computer Science}, pages 151--162. Springer, 1997.

\bibitem[YLL99]{YaoLL99}
Xin Yao, Yong Liu, and Guangming Lin.
\newblock Evolutionary programming made faster.
\newblock {\em {IEEE} Transactions on Evolutionary Computation}, 3:82--102,
  1999.

\end{thebibliography}

\newcommand{\etalchar}[1]{$^{#1}$}

\newpage
\appendix
\section*{Appendix: Computation of Table~\ref{tbl:progress}}

In this appendix we compute all estimates of the true progress probability $p_{d(x)}$ shown in Table~\ref{tbl:progress}.
We use the same expression for estimating $p_{d(x)}$ as in Lemma~\ref{lem:progress-global}, that by Lemma~\ref{lem:progress-lambda} is,

\begin{align*} 
	p_{d(x)} &= \sum_{\lambda = 1}^{u} C_{\beta, u} \lambda^{-\beta} p_{d(x)}(\lambda) \\
	&\ge \begin{cases}
		C_{\beta, u} C \frac{d(x)}{n} \sum_{\lambda = 1}^{u}  \lambda^{2-\beta}, &\text{ if } u \le \sqrt{\frac{n}{d(x)}}, \\
		C_{\beta, u} C \frac{d(x)}{n} \sum_{\lambda = 1}^{\lfloor \sqrt{\frac{n}{d(x)}} \rfloor} \lambda^{2-\beta} + C_{\beta, u} C \sum_{\lambda = \lfloor \sqrt{\frac{n}{d(x)}} \rfloor + 1}^u \lambda^{-\beta}, &\text{ else,}
	\end{cases}
\end{align*} 

where $C$ is some constant. 
Recall that by Lemma~\ref{lem:sum-upper} we have
\begin{itemize}
	\item if $\beta < 0$, then $C_{\beta, u} \ge  u^{\beta - 1} \frac{1 - \beta}{2 - \beta}$,
	\item if $\beta \in [0, 1)$, then $C_{\beta, u} \ge u^{\beta - 1} (1 - \beta)$,
	\item if $\beta = 1$, then $C_{\beta, u} \ge \frac{1}{\ln(u) + 1}$, and
	\item if $\beta > 1$, then $C_{\beta, u} \ge \frac{\beta - 1}{\beta}$.
\end{itemize} 

Now we consider 11 cases depending on $\beta$ and $u$. We start with the cases when $u \le \sqrt{\frac{n}{d(x)}}$ and therefore estimate $p_{d(x)}$ as 
\begin{align*}
	p_{d(x)} &\ge
		C_{\beta, u} C \frac{d(x)}{n} \sum_{\lambda = 1}^{u}  \lambda^{2-\beta}.
\end{align*}

\textbf{Case 1: $\beta < 0$, $u \le \sqrt{\frac{n}{d(x)}}$.}

By Lemma~\ref{lem:sum-lower} we have
\begin{align*} 
	p_{d(x)} &\ge  C C_{\beta, u} \frac{d(x)}{n} \sum_{i = 1}^u \lambda^{2 - \beta} \\
			 &\ge C \cdot u^{\beta - 1} \frac{1 - \beta}{2 - \beta} \cdot \frac{d(x)}{n} \cdot \frac{u^{3 - \beta} - 1}{3 - \beta} = \Omega\left(\frac{d(x)u^2}{n}\right).
\end{align*} 

\textbf{Case 2: $\beta \in [0, 1)$, $u \le \sqrt{\frac{n}{d(x)}}$.}

By Lemma~\ref{lem:sum-lower} we have
\begin{align*} 
	p_{d(x)} &\ge  C C_{\beta, u} \frac{d(x)}{n} \sum_{i = 1}^u \lambda^{2 - \beta} \\
			 &\ge C \cdot u^{\beta - 1} (1 - \beta) \cdot \frac{d(x)}{n}  \cdot \frac{u^{3 - \beta} - 1}{3 - \beta} = \Omega\left(\frac{d(x)u^2}{n}\right),
\end{align*} 
which is the same as in Case 1.

\textbf{Case 3: $\beta = 1$, $u \le \sqrt{\frac{n}{d(x)}}$.}

In this case we have
\begin{align*} 
	p_{d(x)} &\ge  C C_{\beta, u} \frac{d(x)}{n} \sum_{i = 1}^u \lambda \\
			 &\ge C \cdot \frac{1 }{\ln(u) + 1} \cdot \frac{d(x)}{n} \cdot \frac{u(u + 1)}{2} = \Omega\left(\frac{d(x)u^2}{n\log(u)}\right).
\end{align*} 

\textbf{Case 4: $\beta \in (1, 3)$, $u \le \sqrt{\frac{n}{d(x)}}$.}

By Lemma~\ref{lem:sum-lower} we have
\begin{align*} 
	p_{d(x)} &\ge  C C_{\beta, u} \frac{d(x)}{n} \sum_{i = 1}^u \lambda^{2 - \beta} \\
			 &\ge C \cdot \frac{\beta - 1}{\beta} \cdot \frac{d(x)}{n} \cdot \frac{u^{3 - \beta} - 1}{3 - \beta} = \Omega\left(\frac{d(x)u^{3 - \beta}}{n}\right).
\end{align*} 

\textbf{Case 5: $\beta = 3$, $u \le \sqrt{\frac{n}{d(x)}}$.}

By Lemma~\ref{lem:sum-lower} we have
\begin{align*} 
	p_{d(x)} &\ge  C C_{\beta, u} \frac{d(x)}{n} \sum_{i = 1}^u \lambda^{-1} \\
			 &\ge C \cdot \frac{2}{3} \cdot \frac{d(x)}{n} \cdot \ln(u) = \Omega\left(\frac{d(x)\log(u)}{n}\right).
\end{align*} 

\textbf{Case 6: $\beta > 3$, $u \le \sqrt{\frac{n}{d(x)}}$.}

We have
\begin{align*} 
	p_{d(x)} &\ge  C C_{\beta, u} \frac{d(x)}{n} \sum_{i = 1}^u \lambda^{2 - \beta} \\
			 &\ge C \cdot \frac{\beta - 1}{\beta} \cdot \frac{d(x)}{n} \cdot 1 = \Omega\left(\frac{d(x)}{n}\right).
\end{align*}

In the following cases we consider $u > \sqrt{\frac{n}{d(x)}}$, hence we estimate $p_{d(x)}$ as

\begin{align*} 
	p_{d(x)} &\ge
		C_{\beta, u} C \frac{d(x)}{n} \sum_{\lambda = 1}^{\lfloor \sqrt{\frac{n}{d(x)}} \rfloor} \lambda^{2-\beta} + C_{\beta, u} C \sum_{\lambda = \lfloor \sqrt{\frac{n}{d(x)}} \rfloor + 1}^u \lambda^{-\beta} \\
		&= C C_{\beta, u} \left(\frac{d(x)}{n} \sum_{\lambda = 1}^{\lfloor \sqrt{\frac{n}{d(x)}} \rfloor} \lambda^{2-\beta} +  \sum_{\lambda = \lfloor \sqrt{\frac{n}{d(x)}} \rfloor + 1}^u \lambda^{-\beta}\right).
\end{align*} 

In all cases we first estimate the sums in the brackets and then put it into the inequality.

\textbf{Case 7: $\beta < 1$, $u > \sqrt{\frac{n}{d(x)}}$.}

We consider three sub-cases.
\begin{enumerate}
	\item \textbf{When $u \le 2\sqrt{\frac{n}{d(x)}} + 2$ and $\sqrt{\frac{n}{d(x)}} \le 4$.}
	
	In this case we also have $u \le 2 \cdot 4 + 2 = 10$. Hence,
	\begin{align*}
		\frac{d(x)}{n} \sum_{\lambda = 1}^{\lfloor \sqrt{\frac{n}{d(x)}} \rfloor} \lambda^{2-\beta} \ge \frac{d(x)}{n} \ge \frac{1}{16} \ge \frac {u^{1 - \beta}}{16 \cdot 10^{1 - \beta}}.
	\end{align*}
	\item \textbf{When $u \le 2\sqrt{\frac{n}{d(x)}} + 2$ and $\sqrt{\frac{n}{d(x)}} > 4$.}
	
	In this case we have $\sqrt{\frac{n}{d(x)}} \ge \frac{u}{2} - 1$. We also have that $\sqrt{\frac{n}{d(x)}}^{3 - \beta} \ge 4^{3 - \beta} > 2^{4 - \beta}$ (therefore, $(\sqrt{\frac{n}{d(x)}} / 2)^{3 - \beta} > 2$). Hence, by Lemma~\ref{lem:sum-lower} we have 
	\begin{align*}
		\frac{d(x)}{n} \sum_{\lambda = 1}^{\lfloor \sqrt{\frac{n}{d(x)}} \rfloor} \lambda^{2-\beta} &\ge \frac{d(x)}{n} \sum_{\lambda = 1}^{\lceil \sqrt{\frac{n}{d(x)}} - 1 \rceil} \lambda^{2-\beta} \ge \frac{d(x)}{n} \cdot \frac{\left(\sqrt{\frac{n}{d(x)}} - 1\right)^{3 - \beta} - 1}{3 - \beta} \\
		&\ge \frac{d(x)}{n} \cdot \frac{\left(\sqrt{\frac{n}{d(x)}} / 2\right)^{3 - \beta} - 1}{3 - \beta} \ge \frac{d(x)}{n} \cdot \frac{\left(\sqrt{\frac{n}{d(x)}} / 2\right)^{3 - \beta}}{2(3 - \beta)} \\
		&\ge \sqrt{\frac{n}{d(x)}}^{1 - \beta} \frac{1}{2^{4 - \beta}(3 - \beta)} \ge \left(\frac{u}{2} - 1\right)^{1 - \beta} \frac{1}{2^{4 - \beta}(3 - \beta)} \\
		&\ge \frac{u^{1 - \beta}}{2^{(6 - 3 \beta)}(3 - \beta)}.
	\end{align*}
	\item \textbf{When $u > 2\sqrt{\frac{n}{d(x)}} + 2$.}
	
	In the same way as in Lemma~\ref{lem:sum-lower} we estimate a sum via a corresponding integral.
	\begin{align*}
		\sum_{\lambda = \lfloor \sqrt{\frac{n}{d(x)}} \rfloor + 1}^u \lambda^{-\beta} & \ge \int_{\lfloor \sqrt{\frac{n}{d(x)}} \rfloor + 1}^u x^{-\beta} dx \ge \int_{u/2}^u x^{-\beta} dx = u^{1 - \beta} \cdot \frac{1 - 2^{\beta - 1}}{1 - \beta}.
	\end{align*}
\end{enumerate}

Summing up all three cases we have that for each $\beta < 1$ there exists a constant $\gamma_1(\beta) = \min\{\frac{1}{16 \cdot 10^{1 - \beta}}, \frac{1}{2^{(6 - 3 \beta)}(3 - \beta)},  \frac{1 - 2^{\beta - 1}}{1 - \beta}\}$ such that 
\begin{align*}
	\frac{d(x)}{n} \sum_{\lambda = 1}^{\lfloor \sqrt{\frac{n}{d(x)}} \rfloor} \lambda^{2-\beta} +  \sum_{\lambda = \lfloor \sqrt{\frac{n}{d(x)}} \rfloor + 1}^u \lambda^{-\beta} \ge \gamma_1(\beta) \cdot u^{1 - \beta}.
\end{align*}

If $\beta < 0$, we have 
\begin{align*}
	p_{d(x)} &\ge C C_{\beta, u} \gamma_1(\beta)  u^{1 - \beta} \ge C u^{\beta - 1} \frac{1 - \beta}{2 - \beta} \gamma_1(\beta)  u^{1 - \beta} = \Omega(1).
\end{align*}
If $\beta \in [0, 1)$, we have
\begin{align*}
	p_{d(x)} &\ge C C_{\beta, u} \gamma_1(\beta)  u^{1 - \beta} \ge C u^{\beta - 1} (1 - \beta) \gamma_1(\beta)  u^{1 - \beta} = \Omega(1).
\end{align*}

\textbf{Case 8: $\beta = 1$, $u > \sqrt{\frac{n}{d(x)}}$.}
We aim at showing that 
\begin{align*}
	p_{d(x)} \ge C \cdot \left(\frac{1}{36\ln(u)} + \frac{\ln(u) - \ln\left(\sqrt{\frac{n}{d(x)}}\right)}{36\ln(u)}\right).
\end{align*}
Note that in this case we do not use asymptotic notation for estimating $p_{d(x)}$ due to having terms of different signs in the bound above (and thus, the leading constants of these terms are important). However note that as long as $u$ is by a constant times greater than $\sqrt{\frac{n}{d(x)}}$, then the first term is dominant, therefore, this bound is $\Omega(\frac{1}{\log(u)})$. If $u$ is at least $\phi \cdot \sqrt{\frac{n}{d(x)}}$ for some super-constant $\phi$, then this bound is $\Omega(\frac{\log(\phi)}{\log(u)})$.

In this case we have $u > \sqrt{\frac{n}{d(x)}} \ge 1$, hence $u \ge 2$. Therefore, by Lemma~\ref{lem:sum-upper} we have 
\begin{align*}
	C_{1,u} \ge \frac{1}{1 + \ln(u)} = \frac{1}{\ln(u)} \cdot \frac{\ln(u)}{1 + \ln(u)} \ge \frac{1}{\ln(u)} \cdot  \frac{\ln(2)}{\ln(2) + 1} > \frac{1}{3\ln(u)}.
\end{align*}

By the formula for a sum of arithmetic progression and estimating the second sum via a corresponding integral in the same way as in Lemma~\ref{lem:sum-lower}, we have
\begin{align*}
	\frac{d(x)}{n} \sum_{\lambda = 1}^{\lfloor \sqrt{\frac{n}{d(x)}} \rfloor} \lambda &+  \sum_{\lambda = \lfloor \sqrt{\frac{n}{d(x)}} \rfloor + 1}^u \lambda^{-1}  \\
	&\ge \frac{d(x)}{n} \cdot \frac{\lfloor \sqrt{\frac{n}{d(x)}} \rfloor \left(\lfloor \sqrt{\frac{n}{d(x)}} \rfloor + 1\right)}{2} + \int_{\lfloor \sqrt{\frac{n}{d(x)}} \rfloor + 1}^u \frac{dx}{x}
\end{align*}
Since for all $x \ge 1$ we have $\frac{\lfloor x \rfloor}{x} \ge \frac{1}{2}$ and $\frac{\lfloor x \rfloor + 1}{x} \ge 1$, we also have 
\begin{align*}
	\frac{\lfloor \sqrt{\frac{n}{d(x)}} \rfloor \left(\lfloor \sqrt{\frac{n}{d(x)}} \rfloor + 1\right)}{2n/d(x)} \ge \frac{1}{4}.
\end{align*}

Now we consider two sub-cases. First, let $u \le e^2 \sqrt{\frac{n}{d(x)}}$. Then we have
\begin{align*}
	p_{d(x)} \ge C C_{1, u} \cdot \frac{1}{4} \ge \frac{C}{12\ln(u)}. 
\end{align*}

Otherwise, if $u > e^2 \sqrt{\frac{n}{d(x)}}$, then we estimate the integral by
\begin{align*}
	\int_{\lfloor \sqrt{\frac{n}{d(x)}} \rfloor + 1}^u \frac{dx}{x} &\ge \int_{\sqrt{\frac{n}{d(x)}} + 1}^u \frac{dx}{x} = \ln(u) - \ln\left(\sqrt{\frac{n}{d(x)}} + 1\right) \\
	&= \ln(u) - \ln\left(\sqrt{\frac{n}{d(x)}}\right) - \ln\left(1 + \sqrt{\frac{d(x)}{n}}\right) \\
	&\ge \ln(u) - \ln\left(\sqrt{\frac{n}{d(x)}}\right) - \sqrt{\frac{d(x)}{n}} \\
	&\ge \frac{\ln(u) - \ln\left(\sqrt{\frac{n}{d(x)}}\right)}{2} + \frac{\ln(u) - \ln\left(\frac{u}{e^2}\right)}{2} - 1 \\
	&= \frac{\ln(u) - \ln\left(\sqrt{\frac{n}{d(x)}}\right)}{2} + 1 - 1.
\end{align*}
Hence, we conclude
\begin{align*}
	p_{d(x)} &\ge C C_{1, u} \left(\frac{1}{4} + \frac{\ln(u) - \ln\left(\sqrt{\frac{n}{d(x)}}\right)}{2}\right) \\
	&\ge C \cdot \frac{1 + 2\left(\ln(u) - \ln\left(\sqrt{\frac{n}{d(x)}}\right)\right)}{12\ln(u)}.
\end{align*}

We unite the two sub-cases with the following lower bound, which holds both for $u \le e^2 \sqrt{\frac{n}{d(x)}}$ and for $u > e^2 \sqrt{\frac{n}{d(x)}}$.
\begin{align*}
	p_{d(x)} &\ge C \cdot \frac{1 + 2\left(\ln(u) - \ln\left(\sqrt{\frac{n}{d(x)}}\right)\right)}{36\ln(u)} \\
			 &\ge C \cdot \frac{1 + \ln(u) - \ln\left(\sqrt{\frac{n}{d(x)}}\right)}{36\ln(u)}.
\end{align*}

\textbf{Case 9: $\beta \in (1, 3)$, $u > \sqrt{\frac{n}{d(x)}}$.}

We consider three sub-cases
\begin{enumerate}
	\item \textbf{When $\beta \le 2$ and $\sqrt{\frac{n}{d(x)}} \le 2$.}
	
	\begin{align*}
		\frac{d(x)}{n} \sum_{\lambda = 1}^{\lfloor \sqrt{\frac{n}{d(x)}} \rfloor} \lambda^{2-\beta} &\ge \frac{d(x)}{n} = \sqrt{\frac{n}{d(x)}}^{1 - \beta} \cdot \sqrt{\frac{n}{d(x)}}^{\beta - 3} \ge \sqrt{\frac{n}{d(x)}}^{1 - \beta} \cdot \left(\frac{1}{2}\right)^{\beta - 3} \\
		&\ge \sqrt{\frac{n}{d(x)}}^{1 - \beta} \cdot \left(\frac{1}{2}\right)^2 = \frac{1}{4}\sqrt{\frac{n}{d(x)}}^{1 - \beta}.
	\end{align*}



	\item \textbf{When $\beta > 2$ and $\lfloor \sqrt{\frac{n}{d(x)}} \rfloor \le 2^{\frac{1}{3 - \beta}}$.} In this case we also have $\sqrt{\frac{n}{d(x)}} \le 2^{\frac{1}{3 - \beta}} + 1$. Hence, we have
	
	\begin{align*}
		\frac{d(x)}{n} \sum_{\lambda = 1}^{\lfloor \sqrt{\frac{n}{d(x)}} \rfloor} \lambda^{2-\beta} &\ge \frac{d(x)}{n} = \sqrt{\frac{n}{d(x)}}^{1 - \beta} \cdot \sqrt{\frac{n}{d(x)}}^{\beta - 3} \\
		&\ge \sqrt{\frac{n}{d(x)}}^{1 - \beta} \cdot \left(2^{\frac{1}{3 - \beta}} + 1\right)^{\beta - 3} \\
		&\ge \sqrt{\frac{n}{d(x)}}^{1 - \beta} \cdot \left(2^{\left(\frac{1}{3 - \beta} + 1\right)}\right)^{\beta - 3} \\
		&= 2^{\beta - 4}\sqrt{\frac{n}{d(x)}}^{1 - \beta} \ge \frac{1}{4}\sqrt{\frac{n}{d(x)}}^{1 - \beta}.
	\end{align*}

	\item \textbf{When $\beta > 2$ and $\lfloor \sqrt{\frac{n}{d(x)}} \rfloor \ge 2^{\frac{1}{3 - \beta}}$ or when $\beta \le 2$ and $\sqrt{\frac{n}{d(x)}} > 2$.}

	In this case we have both $\lfloor \sqrt{\frac{n}{d(x)}} \rfloor^{3 - \beta} \ge 2$ and $\sqrt{\frac{n}{d(x)}} \ge 2$. Hence, by Lemma~\ref{lem:sum-lower} we have
	\begin{align*}
		\frac{d(x)}{n} \sum_{\lambda = 1}^{\lfloor \sqrt{\frac{n}{d(x)}} \rfloor} \lambda^{2-\beta} &\ge \frac{d(x)}{n} \cdot \frac{\lfloor \sqrt{\frac{n}{d(x)}} \rfloor^{3 - \beta} - 1}{3 - \beta} \ge \frac{d(x)}{n} \cdot \frac{\lfloor \sqrt{\frac{n}{d(x)}} \rfloor^{3 - \beta}}{2(3 - \beta)}\\
		&\ge \frac{d(x)}{n} \cdot \frac{\left(\sqrt{\frac{n}{d(x)}} - 1 \right)^{3 - \beta}}{2(3 - \beta)} \ge \frac{d(x)}{n} \cdot \frac{\left(\frac{1}{2}\sqrt{\frac{n}{d(x)}}\right)^{3 - \beta}}{2(3 - \beta)} \\
		&\ge \sqrt{\frac{n}{d(x)}}^{1 - \beta} \frac{1}{2^{4 - \beta}(3 - \beta)}.
	\end{align*}
\end{enumerate}

Summing up all three cases we have that for each $\beta \in (1, 3)$ there exists a constant $\gamma_2(\beta) = \min\{\frac{1}{4},  \frac{1}{2^{4 - \beta}(3 - \beta)}\}$ such that 
\begin{align*}
	\frac{d(x)}{n} \sum_{\lambda = 1}^{\lfloor \sqrt{\frac{n}{d(x)}} \rfloor} \lambda^{2-\beta} \ge \gamma_2(\beta) \cdot \sqrt{\frac{n}{d(x)}}^{1 - \beta}.
\end{align*}

Taking into account that $C_{\beta, u} \ge \frac{\beta - 1}{\beta}$, we obtain
\begin{align*}
	p_{d(x)} \ge C C_{\beta, u} \gamma(\beta) \sqrt{\frac{n}{d(x)}}^{1 - \beta} = \Omega\left(\sqrt{\frac{n}{d(x)}}^{1 - \beta}\right).
\end{align*}

\textbf{Case 10: $\beta = 3$, $u > \sqrt{\frac{n}{d(x)}}$.}

If $\sqrt{\frac{n}{d(x)}} \ge 2$, we compute
\begin{align*}
	p_{d(x)} &\ge C C_{3, u} \frac{d(x)}{n} \sum_{\lambda = 1}^{\lfloor \sqrt{\frac{n}{d(x)}} \rfloor} \lambda^{-1} \ge C \cdot \frac{2}{3} \cdot \frac{d(x)}{n} \ln\left(\lfloor \sqrt{\frac{n}{d(x)}} \rfloor\right) \\
	&= \Omega\left(\frac{\ln\left(\sqrt{\frac{n}{d(x)}}\right)}{n/d(x)}\right).
\end{align*}
Otherwise,
\begin{align*}
	p_{d(x)} &\ge C C_{3, u} \frac{d(x)}{n} = \Omega\left(\frac{1}{n/d(x)}\right).
\end{align*}

Therefore,
\begin{align*}
	p_{d(x)} & = \Omega\left(\frac{\ln\left(\sqrt{\frac{n}{d(x)}}\right) + 1}{n/d(x)}\right).
\end{align*}

\textbf{Case 11: $\beta > 3$, $u > \sqrt{\frac{n}{d(x)}}$.}

In this case we have
\begin{align*} 
	p_{d(x)} &\ge  C C_{\beta, u} \frac{d(x)}{n} \sum_{\lambda = 1}^{\lfloor \sqrt{\frac{n}{d(x)}} \rfloor} \lambda^{2 - \beta} \\
			 &\ge C \cdot \frac{d(x)}{n} \cdot \frac{\beta - 1}{\beta} \cdot 1 = \Omega\left(\frac{d(x)}{n}\right).
\end{align*}

\section*{Appendix: Computation of Table~\ref{tbl:runtimes}}

In this appendix we compute the values of the expected runtime shown in Table~\ref{tbl:runtimes}. We start with computing the expected runtimes in terms of iterations for each value of the algorithm's meta-parameter $\beta$. Recall that $p_d$ is the probability to create a better offspring in one iteration, which is shown in Table~\ref{tbl:progress}. Hence, using the fitness levels argument we can estimate the expected number of iterations before we find the optimum as follows.

\begin{align*}
	E[T_I] \le \sum_{d = 1}^{n} \frac{1}{p_d} = \sum_{d = 1}^{\lfloor \frac{n}{u^2} \rfloor} \frac{1}{p_d} + \sum_{d = \lfloor \frac{n}{u^2} \rfloor + 1}^{n} \frac{1}{p_d}.
\end{align*}

Note that in the first sum we have $u \le \sqrt{\frac{n}{d}}$ (thus, we should use values for $p_d$ from the left column of Table~\ref{tbl:progress}) and in the second sum we have $u > \sqrt{\frac{n}{d}}$ (thus, we should use the estimates from the right column). Note that $p_d = \Omega(f(n, d, u))$ in Table~\ref{tbl:progress} means that for each $\beta$ there exists a constant $\gamma(\beta)$ (independent of $n$, $d$ and $u$) such that $p_d \ge \gamma(\beta) \cdot f(n, d, u)$. We will use this constant in our further computations.

To estimate the expected runtime we consider five cases.

\textbf{Case 1: $\beta < 1$.}

In this case we have 
\begin{align*}
	E[T_I] &\le \sum_{d = 1}^{\lfloor \frac{n}{u^2} \rfloor} \frac{1}{p_d} + \sum_{d = \lfloor \frac{n}{u^2} \rfloor + 1}^{n} \frac{1}{p_d} \\
	&\le \frac{1}{\gamma(\beta)} \left(\sum_{d = 1}^{\lfloor \frac{n}{u^2} \rfloor} \frac{n}{du^2} + \sum_{d = \lfloor \frac{n}{u^2} \rfloor + 1}^{n} 1 \right) \\
	&\le \frac{1}{\gamma(\beta)} \left(\frac{n}{u^2} \left(\ln\lfloor \frac{n}{u^2} \rfloor + 1\right) + n - \lfloor \frac{n}{u^2} \rfloor\right) \\
	&= O\left(\frac{n}{u^2}\ln\left(\frac{n}{u^2}\right) + n\right),
\end{align*}
where we used the estimates for the sums from Lemma~\ref{lem:sum-upper}. Note that when $u \ge \sqrt{\ln(n)}$, we have 
\begin{align*}
	\frac{n}{u^2}\ln\left(\frac{n}{u^2}\right) \le \frac{n}{\ln(n)} \ln(n) = O(n).
\end{align*}
Otherwise, we have
\begin{align*}
	\frac{n}{u^2}\ln\left(\frac{n}{u^2}\right) \ge \frac{n}{\ln(n)} (\ln(n) - \ln\ln(n)) = \Omega(n).
\end{align*}
Therefore, we conclude
\begin{align*}
	E[T_I] = \begin{cases}
		O\left(\frac{n}{u^2}\ln\left(\frac{n}{u^2}\right)\right), &\text{ if } u < \sqrt{\ln(n)}, \\
		O(n), &\text{ if } u \ge \sqrt{\ln(n)}.
	\end{cases}
\end{align*}

\textbf{Case 2: $\beta = 1$.}
In this case we have
\begin{align*}
	E[T_I] &\le \sum_{d = 1}^{\lfloor \frac{n}{u^2} \rfloor} \frac{1}{p_d} + \sum_{d = \lfloor \frac{n}{u^2} \rfloor + 1}^{n} \frac{1}{p_d}\\
	&\le \frac{1}{\gamma(\beta)} \left(\sum_{d = 1}^{\lfloor \frac{n}{u^2} \rfloor} \frac{n\ln(u)}{du^2} + \sum_{d = \lfloor \frac{n}{u^2} \rfloor + 1}^{n} \frac{\ln(u)}{1 + \ln(u) - \ln(\sqrt{\frac{n}{d}})} \right) \\
	&\le \frac{1}{\gamma(\beta)} \left(\frac{n\ln(u)}{u^2} \left(\ln\lfloor \frac{n}{u^2} \rfloor + 1\right) + \sum_{d = \lfloor \frac{n}{u^2} \rfloor + 1}^{\lfloor \frac{n}{u} \rfloor} \ln(u) + \sum_{d = \lfloor \frac{n}{u} \rfloor + 1}^{n} \frac{\ln(u)}{1 + \frac{1}{2}\ln(u)} \right) \\
	&\le \frac{1}{\gamma(\beta)} \left(\frac{n\ln(u)}{u^2} \left(\ln\lfloor \frac{n}{u^2} \rfloor + 1\right) + \frac{n\ln(u)}{u} + n \cdot \frac{\ln(u)}{1 + \frac{1}{2}\ln(u)} \right) \\
	&= O\left(\frac{n\log(u)\log(\frac{n}{u^2})}{u^2} + \frac{n\log(u)}{u} + n\right)
\end{align*}

Note that $\frac{n\ln(u)\ln(\frac{n}{u^2})}{u^2}$ is a decreasing function of $u$ for all $u \ge 1$, which can be shown by considering its derivative (we omit this tedious computation). Hence, if $u < \sqrt{\ln(n)\ln\ln(n)}$, then we have
\begin{align*}
	\frac{n\ln(u)\ln(\frac{n}{u^2})}{u^2} &\ge \frac{n(\ln\ln(n) + \ln\ln\ln(n))(\ln(n) - \ln\ln(n) - \ln\ln\ln(n))}{2\ln(n)\ln\ln(n)} \\
	&= \Omega(n).
\end{align*}
For such $u$ we also have
\begin{align*}
	\frac{n\ln(u)\ln(\frac{n}{u^2})}{u^2} &\ge \frac{n\ln(u)}{u} \cdot \frac{\ln(\frac{n}{u^2})}{u} \\
	&\ge \frac{n\ln(u)}{u} \cdot \frac{(\ln(n) - \ln\ln(n) - \ln\ln\ln(n))}{\sqrt{\ln(n)\ln\ln(n)}} \\
	&= \Omega\left(\frac{n\log(u)}{u}\sqrt{\frac{\log(n)}{\log\log(n)}}\right) = \Omega\left(\frac{n\log(u)}{u}\right).
\end{align*}
If $u \ge \sqrt{\ln(n)\ln\ln(n)}$, we have
\begin{align*}
	\frac{n\ln(u)\ln(\frac{n}{u^2})}{u^2} \le \frac{n\ln\ln(n)(\ln(n) - \ln\ln(n) - \ln\ln\ln(n))}{2\ln(n)\ln\ln(n)} = O(n),
\end{align*}
and we have
\begin{align*}
	\frac{n\ln(u)}{u} \le n = O(n).
\end{align*}

Hence, we conclude
\begin{align*}
	E[T_I] = \begin{cases}
		O\left(\frac{n\log(u)}{u^2}\log\left(\frac{n}{u^2}\right)\right), &\text{ if } u < \sqrt{\ln(n)\ln\ln(n)}, \\
		O(n), &\text{ if } u \ge \sqrt{\ln(n)\ln\ln(n)}.
	\end{cases}
\end{align*}

\textbf{Case 3: $\beta \in (1, 3)$.}

In this case we have 
\begin{align*}
	E[T_I] &\le \sum_{d = 1}^{\lfloor \frac{n}{u^2} \rfloor} \frac{1}{p_d} + \sum_{d = \lfloor \frac{n}{u^2} \rfloor + 1}^{n} \frac{1}{p_d}\\
	&\le \frac{1}{\gamma(\beta)} \left(\sum_{d = 1}^{\lfloor \frac{n}{u^2} \rfloor} \frac{n}{du^{3 - \beta}} + \sum_{d = \lfloor \frac{n}{u^2} \rfloor + 1}^{n} \sqrt{\frac{n}{d}}^{\beta - 1} \right) \\
	& \le \frac{1}{\gamma(\beta)} \left(\frac{n}{u^{3 - \beta}}\left(\ln\left(\frac{n}{u^2}\right) + 1\right) + n^{(\beta - 1)/2}\sum_{d = 1}^{n - 1} d^{(1 - \beta)/2} \right) \\
	&\le \frac{1}{\gamma(\beta)} \left(\frac{n}{u^{3 - \beta}}\left(\ln\left(\frac{n}{u^2}\right) + 1\right) + n^{(\beta - 1)/2} \cdot \frac{n^{(3 - \beta)/2} - 1}{(3 - \beta)/2} \right) \\
	&= O\left(\frac{n}{u^{3 - \beta}}\log\left(\frac{n}{u^2}\right) + n\right),
\end{align*}
where we used Lemma~\ref{lem:sum-upper} to estimate the sums.
When $u < (\ln(n))^{1/(3 - \beta)}$, we have 
\begin{align*}
	\frac{n}{u^{3 - \beta}}\ln\left(\frac{n}{u^2}\right) \ge \frac{n(\ln(n) - \frac{2}{3 - \beta}\ln\ln(n))}{\ln(n)} = \Omega(n).
\end{align*}
Otherwise, we have
\begin{align*}
	\frac{n}{u^{3 - \beta}}\ln\left(\frac{n}{u^2}\right) \le \frac{n\ln(n)}{\ln(n)} = n.
\end{align*}

Therefore, we have 
\begin{align*}
	E[T_I] = \begin{cases}
		O\left(\frac{n}{u^{3 - \beta}}\log\left(\frac{n}{u^2}\right)\right), &\text{ if } u < (\ln(n))^{1/(3 - \beta)}, \\
		O(n), &\text{ if } u \ge (\ln(n))^{1/(3 - \beta)}.
	\end{cases}
\end{align*}

\textbf{Case 4: $\beta = 3$.}
We compute
\begin{align*}
	E[T_I] &\le \sum_{d = 1}^{\lfloor \frac{n}{u^2} \rfloor} \frac{1}{p_d} + \sum_{d = \lfloor \frac{n}{u^2} \rfloor + 1}^{n} \frac{1}{p_d}\\
	&\le \frac{1}{\gamma(\beta)} \left(\sum_{d = 1}^{\lfloor \frac{n}{u^2} \rfloor} \frac{n}{d\ln(u)} + \sum_{d = \lfloor \frac{n}{u^2} \rfloor + 1}^{n} \frac{n}{d\left(\ln\left(\frac{n}{d}\right) + 1\right)} \right) \\
	& \le \frac{1}{\gamma(\beta)} \left(\frac{n}{\ln(u)}\left(\ln\left(\frac{n}{u^2}\right) + 1\right) + n + n \sum_{d = \lfloor \frac{n}{u^2} \rfloor + 2}^{n} \frac{1}{d\left(\ln\left(\frac{n}{d}\right) + 1\right)} \right) \\
	& \le \frac{1}{\gamma(\beta)} \left(\frac{n}{\ln(u)}\left(\ln\left(\frac{n}{u^2}\right) + 1\right) + n + n \int_{n/u^2}^n \frac{dx}{x(\ln(n) - \ln(x) + 1)} \right), \\
\end{align*}
where we used the fact that $f(x) = \frac{1}{x(\ln(n) - \ln(x) + 1)}$ is a decreasing function in interval $[1, n]$ to estimate the sum via a corresponding integral.
We estimate the integral as follows.
\begin{align*}
	\int_{n/u^2}^n \frac{dx}{x(\ln(n) - \ln(x) + 1)} &= - \int_{n/u^2}^n \frac{d(\ln(n) - \ln(x) + 1)}{(\ln(n) - \ln(x) + 1)} \\
	&= \ln((\ln(n) - \ln(x) + 1))\bigg\rvert_n^{n/u^2} \\
	&= \ln(2\ln(u) + 1) 
\end{align*}
Therefore,
\begin{align*}
	E[T_I] &\le \frac{1}{\gamma(\beta)} \left(\frac{n}{\ln(u)}\left(\ln\left(\frac{n}{u^2}\right) + 1\right) + n (\ln(2\ln(u) + 1) + 1) \right) \\
	&= O\left(\frac{n}{\log(u)}\log\left(\frac{n}{u^2}\right) + n\log\log(u)\right)
\end{align*}
Note that the first term is decreasing in $u$, while the second one is increasing. We show that they are asymptotically the same when $u = n^{1/\ln\ln(n)}$.
\begin{align*}
	\frac{n}{\ln(n^{1/\ln\ln(n)})}\ln\left(\frac{n}{n^{2/\ln\ln(n)}}\right) &= \frac{n \ln\ln(n)}{\ln(n)} \cdot \left(\ln(n) - \frac{2\ln(n)}{\ln\ln(n)}\right) \\
	&= \Theta(n \ln\ln(n)),
\end{align*}
\begin{align*}
	n\ln\ln(n^{1/\ln\ln(n)}) &= n\ln\frac{\ln(n)}{\ln\ln(n)} = n\ln\ln(n) - n\ln\ln\ln(n) \\
	&= \Theta(n \ln\ln(n)).
\end{align*}
Therefore, when $u \le n^{1/\ln\ln(n)}$, the first term is dominant, otherwise the second term is dominant. Hence, we conclude
\begin{align*}
	E[T_I] = \begin{cases}
		O\left(\frac{n}{\log(u)}\log\left(\frac{n}{u^2}\right)\right), &\text{ if } u < n^{1/\ln\ln(n)}, \\
		O(n\log\log(u)), &\text{ if } u \ge n^{1/\ln\ln(n)}.
	\end{cases}
\end{align*}

\textbf{Case 5: $\beta > 3$.}

In this case we have
\begin{align*}
	E[T_I] &\le \sum_{d = 1}^{n} \frac{1}{p_d} \le \frac{1}{\gamma(\beta)} \sum_{d = 1}^n \frac{n}{d} \le \frac{n(\ln(n) + 1)}{\gamma(\beta)} = O(n\log(n))\\
\end{align*}

We complete the computation of the right column of Table~\ref{tbl:runtimes} by using Wald's equation (Lemma~\ref{lem:wald}) and estimates of the expected cost of each iteration shown in Lemma~\ref{lem:exp-lambda}.

\textbf{Case 1: $\beta < 1$.}

If $u \ge \sqrt{\ln(n)}$, then
\begin{align*}
	E[T_F] = O(n) \cdot \Theta(u) = O(nu).
\end{align*}

If $u < \sqrt{\ln(n)}$, then
\begin{align*}
	E[T_F] = O\left(\frac{n}{u^2}\log\frac{n}{u^2}\right) \cdot \Theta(u) = O\left(\frac{n}{u}\log\frac{n}{u^2}\right).
\end{align*}

\textbf{Case 2: $\beta = 1$.}

If $u \ge \sqrt{\ln(n)\ln\ln(n)}$, then
\begin{align*}
	E[T_F] = O(n) \cdot \Theta\left(\frac{u}{\log(u)}\right) = O\left(\frac{nu}{\log(u)}\right).
\end{align*}

If $u < \sqrt{\ln(n)\ln\ln(n)}$, then
\begin{align*}
	E[T_F] = O\left(\frac{n\log(u)}{u^2}\log\frac{n}{u^2}\right) \cdot \Theta\left(\frac{u}{\log(u)}\right) = O\left(\frac{n}{u}\log\frac{n}{u^2}\right).
\end{align*}

\textbf{Case 3: $\beta \in (1, 2)$.}

If $u \ge (\ln(n))^{1/(3 - \beta)}$, then
\begin{align*}
	E[T_F] = O(n) \cdot \Theta(u^{2 - \beta}) = O(nu^{2 - \beta}).
\end{align*}

If $u < (\ln(n))^{1/(3 - \beta)}$, then
\begin{align*}
	E[T_F] = O\left(\frac{n}{u^{3 - \beta}}\log\frac{n}{u^2}\right) \cdot \Theta(u^{2 - \beta}) = O\left(\frac{n}{u}\log\frac{n}{u^2}\right).
\end{align*}

\textbf{Case 4: $\beta = 2$.}

If $u \ge \ln(n)$, then
\begin{align*}
	E[T_F] = O(n) \cdot \Theta(\log(u)) = O(n\log(u)).
\end{align*}

If $u < \ln(n)$, then
\begin{align*}
	E[T_F] = O\left(\frac{n}{u}\log\frac{n}{u^2}\right) \cdot \Theta(\log(u)) = O\left(\frac{n\log(u)}{u}\log\frac{n}{u^2}\right).
\end{align*}

\textbf{Case 5: $\beta \in (2, 3)$.}

If $u \ge (\ln(n))^{1/(3 - \beta)}$, then
\begin{align*}
	E[T_F] = O(n) \cdot \Theta(1) = O(n).
\end{align*}

If $u < (\ln(n))^{1/(3 - \beta)}$, then
\begin{align*}
	E[T_F] = O\left(\frac{n}{u^{3 - \beta}}\log\frac{n}{u^2}\right) \cdot \Theta(1) = O\left(\frac{n}{u^{3 - \beta}}\log\frac{n}{u^2}\right).
\end{align*}

\textbf{Case 6: $\beta = 3$.}

If $u \ge n^{1/\ln\ln(n)}$, then
\begin{align*}
	E[T_F] = O(n\log\log(u)) \cdot \Theta(1) = O(n\log\log(u)).
\end{align*}

If $u < n^{1/\ln\ln(n)}$, then
\begin{align*}
	E[T_F] = O\left(\frac{n}{\log(u)}\log\frac{n}{u^2}\right) \cdot \Theta(1) = O\left(\frac{n}{\log(u)}\log\frac{n}{u^2}\right).
\end{align*}

\textbf{Case 7: $\beta > 3$.}

For all $u$ we have
\begin{align*}
	E[T_F] = O(n\log(n)) \cdot \Theta(1) = O(n\log(n)).
\end{align*}

}

\end{document}